\documentclass[english]{article}
\usepackage[left=3cm, right=3cm, top=2cm, bottom=2cm]{geometry}
\usepackage[T1]{fontenc}
\usepackage[utf8]{inputenc}
\usepackage{xcolor}
\usepackage{verbatim}
\usepackage{float}
\usepackage{fancybox}
\usepackage{calc}
\usepackage{mathtools}
\usepackage{bm}
\usepackage{algorithm}
\usepackage{algpseudocode}
\usepackage{amsmath}
\usepackage{amsthm}
\usepackage{amssymb}
\usepackage{graphicx}
\usepackage{hyperref}
\usepackage{array}
\usepackage{subcaption}
\usepackage{authblk}

\usepackage{thmtools, thm-restate}
\declaretheorem{theorem}
\declaretheorem{lemma}

\declaretheorem{definition}

\definecolor{green}{rgb}{0.0, 0.5, 0.0}
\definecolor{green3}{rgb}{0,0.6,0}
\definecolor{indigo}{rgb}{0.3,0.0,0.5}

\DeclareMathOperator*{\argmin}{arg\,min}
\DeclareMathOperator*{\minimize}{minimize~~}

\newcommand{\Mat}[1]{\bm{#1}}
\newcommand{\MH}[0]{\Mat{H}}
\newcommand{\MW}[0]{\Mat{W}}

\newcommand{\MY}[0]{\Mat{Y}}

\newcommand{\Rbb}[0]{\mathbb{R}}

\makeatother

\begin{document}
\title{Efficient algorithms for regularized Poisson Non-negative Matrix Factorization}

\author[1]{Nathana\"el Perraudin}
\author[2]{Adrien Teutrie}
\author[3]{Cécile Hébert}
\author[1]{Guillaume Obozinski}
\affil[1]{Swiss Data Science Center, EPFL and ETH Z\"urich, Andreasstrasse 5, Z\"urich, 8050, Z\"urich, Switzerland}
\affil[2]{Unité Matériaux et Transformations, UMR-CNRS 8207, Université de Lille, Cité scientifique, Bâtiment C6, Villeneuve d’Ascq, 59655, Nord, France}
\affil[3]{Electron Spectrometry and Microscopy Laboratory, Institute of Physics, EPFL, Bâtiment PH, Station 3, Lausanne, 1015, Vaud, Switzerland}

\date{}
\setcounter{Maxaffil}{0}
\renewcommand\Affilfont{\itshape\small}

\maketitle

\begin{abstract}
We consider the problem of regularized Poisson Non-negative Matrix Factorization (NMF) problem, encompassing various regularization terms such as Lipschitz and relatively smooth functions, alongside linear constraints. This problem holds significant relevance in numerous Machine Learning applications, particularly within the domain of physical linear unmixing problems. A notable challenge arises from the main loss term in the Poisson NMF problem being a KL divergence, which is non-Lipschitz, rendering traditional gradient descent-based approaches inefficient. In this contribution, we explore the utilization of Block Successive Upper Minimization (BSUM) to overcome this challenge. We build approriate majorizing function for Lipschitz and relatively smooth functions, and show how to introduce linear constraints into the problem. This results in the development of two novel algorithms for regularized Poisson NMF. We conduct numerical simulations to showcase the effectiveness of our approach.
\end{abstract}

\paragraph*{Disclaimer}
This document is a technical report and has not undergone peer review. 
The findings and conclusions presented herein are solely based on the authors' research and analysis.
We apologize for any potential errors or shortcomings in the content.

\section{Introduction}
The problem of factorizing a matrix $\MY\approx\MW\MH$ as Non Negative components $\MW\geq0,\MH\geq0$ is central in many Machine Learning (ML) applications~\cite{smaragdis2003non,fevotte2009nonnegative,brunet2004metagenes}.
The motivation for performing such a factorization is that $\MY$ is often associated with a probability distribution density of the form $P_{\MY}(\MW,\MH)=\tilde{P}_{\MY}(\MW\MH)$.
Typically, the optimal decomposition is found by minimizing the negative log-likelihood of that distribution:
\begin{equation}
\minimize_{\MW\geq0,\MH\geq0}-\sum_{i,j}\log\left(P_{\MY}\left(\MW,\MH\right)\right)
\label{eq:general-loss}
\end{equation}
If $\MY$ is assumed to be perturbed with Normal noise, we obtain a Gaussian distribution, i.e. $P_{\MY}\left(\MW,\MH\right)\propto e^{-\|\MW\MH-\MY\|^{2}}$, and we end up with the classic non-negative matrix factorization (NMF) problem~\cite{lee2001algorithms,lee1999learning}, where the quadratic function$\|\MW\MH-\MY\|^{2}$ is minimized. 
For other distribution families, and in particular exponential families, a $\log$ term often appears. In particular, the Poisson negative log-likelihood model~\cite{lee2001algorithms} leads to a loss of the form
\begin{equation}
\mathcal{L}_{\MY}\left(\MW,\MH\right):=-\left\langle \MY,\log\left(\MW\MH\right)\right\rangle +\left\langle \mathbf{1},\MW\MH\right\rangle \propto-\sum_{i,j}\log\left(P_{\MY}\left(\MW,\MH\right)\right),\label{eq:general-poisson-loss}
\end{equation}
where the inner product over matrices is the Frobenius inner product defined as $\left\langle \Mat{A},\Mat{B}\right\rangle =\sum_{ij}a_{ij}b_{ij}$. 

\paragraph{Regularized Poisson Non Negative Matrix Factorisation}
In many problems (see, for example, \cite{teurtrie2023espm, teboulle2020novel,kannan2018deep,teurtrie2024stem}), additional prior information about the matrices $\MW,\MH$ is known. For example, it might be known that the columns of $\MH$ are smooth, or that the rows of $\MW$ are sparse. One might also be interested in normalizing to unity the columns of $\MH$ or the rows of $\MW$ because they might quantify physical quantities for which normalization is necessary. For example, in the analysis of hyperspectral imaging data, the images $\MH$ are assumed to be smooth and the components $\MW$ are summing to the unity \cite{teurtrie2023espm, teurtrie2024stem}.
Typically, this information can be encoded via an extra regularization term $R\left(\MW,\MH\right)$ and/or additional constraints $\MW\in\mathcal{C}_{1},\MH\in\mathcal{C}_{2}$.
This leads to the general optimization problem we solve in this contribution:
\begin{equation}
\begin{split}
\minimize_{\MW ,\MH}~~ & \mathcal{L}_{\MY}\left(\MW,\MH\right){\color{brown}+R_{W}\left(\MW\right)+R_{H}\left(\MH\right),} \\
\text{subject to}~~ & {\color{violet}\MW\geq\epsilon,\MH\geq\epsilon},~~{\color{cyan}\bm{e}_{H}^{\top}\MH=\bm{1} \text{ or }\MW\bm{e}_{W}=\bm{1}^{\top}}
\end{split}
\label{eq:general-problem}
\end{equation}
where $\epsilon>0$. The different colors emphasize the changes compared to the traditional problem of \cite{lee2001algorithms}. First, in violet, we slightly simplify the problem by imposing strict non-negtativity. While, this is not strictly necessary\footnote{The case $\epsilon=0$ could be handled with an approach similar to~\cite{lin2007convergence}.}, this assumption significantly simplify our analysis. We believe that handling $\epsilon=0$ is an unnecessary complication, as the results with $\epsilon$ close to machine precision will be practically identical. Second, in light blue, we consider the case where the constraints are linear, specifically $\bm{e}_{H}^{\top}\MH=\bm{1}$ or $\MW\bm{e}_{W}=\bm{1}^{\top}$.
In general, it is only meaningful to use one of the constraints, as it will fix the ratio between $\MW$ and $\MH$. We note that this includes the simplex constraint when $\bm{e}=\bm{1}.$ 
Third, in brown, we consider regularizations $R_{W}\left(\MW\right)+R_{H}\left(\MH\right)$ of the form:
\begin{equation}
r\left(\bm{x}\right)=s_{L}\left(\bm{x}\right)+s_{R}\left(\bm{x}\right)+\sum_{j=1}^{n}s_{C}\left(x_{j}\right),\label{eq:regularisation-general}
\end{equation}
where $\bm{x}$ is the vector of a row of $\MW$ or a colum of $\bm{H}$., i.e $R_{W}\left(\MW\right)=\sum_{i}r_{H}\left(\bm{w}_{i}\right)$ and $R_{H}\left(\MH\right)=\sum_{j}r_{W}\left(\bm{h}_{i}\right).$


\begin{enumerate}
\item 
The term $s_{L}$ is assumed to be $\sigma_{L}$ gradient Lipschitz, i.e., there exists a $\sigma_{L}$ such that
\begin{equation}
\|\nabla s_{L}\left(\bm{x}\right)-\nabla s_{L}\left(\bm{y}\right)\|_{2}\leq\sigma_{L}\|\bm{x}-\bm{y}\|_{2}\hspace{1em}\text{for all }\bm{x},\bm{y}\in\mathcal{C}.
\label{eq:gradient-lipschitz}
\end{equation}
Alternatively, this condition could be rewriten as 
\[
s_{L}\left(\bm{y}\right) \leq s_{L} \left(\bm{x}\right) 
+ \left\langle \nabla s_L \left(\bm{x}\right),\bm{y}-\bm{x}\right\rangle
 + \frac{\sigma_{L}}{2} \|\bm{x} -\bm{y} \|_{2}^2 \hspace{1em}\text{for all }\bm{x},\bm{y}\in\mathcal{C}.
\]
In this contribution, we will consider in particular $s_{L}\left(\bm{x}\right)=\bm{x}^{\top}\Delta\bm{x}$,
where $\Delta$ is the Laplacian operator, favoring smoothness in the columns of $\MH$. In this case $\sigma_L=2 \lambda_{\text{max}}(\Delta)$.
\item
The term $s_{R}\left(\bm{x}\right)$ is assumed to be $\sigma_{R}$ relatively smooth with respect to $\kappa\left(\bm{x}\right)=-\bm{1}^{\top}\log\left(\bm{x}\right)$. Relative smoothness is a generalization of Lipschitz smoothness, and is defined as follows~\cite[Definition 1.1]{lu2018relatively}:
\begin{equation}
f\left(\bm{y}\right)\leq f\left(\bm{x}\right)+\left\langle \nabla f\left(\bm{x}\right),\bm{y}-\bm{x}\right\rangle +\sigma_{R}\mathcal{B}_{\kappa}\left(\bm{y},\bm{x}\right) \label{eq:relative-smoothness}
\end{equation}
where $\mathcal{B}_{\kappa}$ is the Bregman divergence~\cite{bregman1967relaxation} associated with $\kappa$~\cite[equation 7]{lu2018relatively}:
\[
\mathcal{B}_{\kappa}\left(\bm{y},\bm{x}\right):=\kappa\left(\bm{y}\right)-\kappa\left(\bm{x}\right)-\left\langle \nabla\kappa\left(\bm{x}\right),\bm{y}-\bm{x}\right\rangle \hspace{1em}\text{for all }\bm{x},\bm{y}\in\mathcal{C}.
\]
While Lipschitz functions can be upper bounded with quadratic functions, we observe from \eqref{eq:relative-smoothness} that relative smoothness allows us to upper bound a function with the Bregman divergence of a function $\kappa$. This allows us to use a much wider range of functions to regularize our problem, and in particular non-gradient Lipschitz ones.
As an example, let us consider $\kappa\left(\bm{x}\right)=-\bm{1}^{\top}\log\left(\bm{x}\right)$, with its Bregman divergence:
\[
\mathcal{B}_{\kappa}\left(\bm{y},\bm{x}\right)=\sum_{i}\left(\frac{y_{i}}{x_{i}}-\log\left(\frac{y_{i}}{x_{i}}\right)-1\right)
\]
One can observe that the objective function \eqref{eq:general-poisson-loss} is relatively smooth with respect to $\kappa\left(\bm{x}\right)=-\bm{1}^{\top}\log\left(\bm{x}\right)$.
This term could also be used to introduce soft contraints such as a log-barrier:
$s_{R}\left(\bm{x}\right)=-\bm{1}^{\top}\log\left(\bm{x}-\epsilon\bm{1}\right)$
for $\bm{x}>\epsilon$ and $+\infty$ otherwise. 
\item 
Eventually, $s_{C}\left(x\right)$ is a smooth point-wise concave function (i.e. $-s_C$ is convex). A typical example of this regularisation could be $s_{C}\left(x\right)=\log\left(x+\alpha^{-1}\right)$
which favor sparsity in the vector $\bm{x}$ without penalizing large values too heavily. Its slope starts at $\alpha$ for $x=0$ and
tends to $0$ for $x\rightarrow\infty$. 
\end{enumerate}
Our approach can handle regularizations of the form $R\left(\MW,\MH\right).$ However, for simplicity of notation, we restrict ourselves to separable regularizations.

The fidelity term $\mathcal{L}_{\MY}\left(\MW,\MH\right)$ is convex in $\MH$ and in $\MW$, but jointly non convex. We note that depending on the regularization and the constraints, multiple equivalent scaled solutions (stationary points) could exist, i.e., $\MW\MH=\MW^{\prime}\MH^{\prime}$ for $\MW^{\prime}=\alpha^{-1}\MW$ and $\MH^{\prime}=\alpha\MH$. 
However, this is generally not the case with additional constraints or regularizations.

\paragraph{Why is this problem challenging?}
In general, Poisson Non Negative Matrix factorization, i.e. minimizing $\mathcal{L}_{\MY}\left(\MW,\MH\right)$ is challenging because it is not
gradient Lipschitz\footnote{The function $\mathcal{L}_{\MY}$ is, according to the definition, gradient Lipschitz because of the constraint $\MW,\MH\geq\epsilon$. However, in practice, the constant $\epsilon$ is chosen to be so small that its actual Lipschitz constant is too large to be useful.} despite being differentiable for $\MW,\MH>0$.
In practice, this implies that there exists no fixed learning rate that ensures convergence of gradient descent and that line search would have to be used.
For that reason, solving the more general problem \eqref{eq:general-problem} is a difficult task, and to the best of our knowledge, there are no existing algorithms that can be directly applied to it. 
Although there exist many algorithms to solve the traditional Poisson NMF problem~\cite{hien2021algorithms,lee1999learning,lee2001algorithms,kim2014algorithms,gillis2014and}, none of them focuses on the regularized case (see the related work Section~\ref{sec:Related-work}). 
Our main contribution is to fill this gap by providing multiple algorithms that minimize \eqref{eq:general-problem} for a wide range of regularizations$R\left(\MW,\MH\right)$ described by \eqref{eq:regularisation-general} and some linear constraints.

\paragraph{Our approach}
A natural approach to minimize \eqref{eq:general-problem} is to optimize for each variable $\MW,\MH$ at a time. 
For example, Block Coordinate Descent (BCD) can be expressed as
\begin{equation}
\MW^{t+1} \leftarrow \minimize_{\MW \in \mathcal{C}_{\MW}}\mathcal{L}\left(\MW,\MH^{t}\right),\label{eq:fullminupdateW}
\end{equation}
\begin{equation}
\MH^{t+1} \leftarrow \minimize_{\MH \in \mathcal{C}_{\MH}}\mathcal{L}\left(\MW^{t+1},\MH\right).\label{eq:fullminupdateH}
\end{equation}
This type of iterative scheme ensures that the loss does not increase between iterations and has been successfully used for the L2 case. However, in the Poisson case, there is no closed-form solution for problems \eqref{eq:fullminupdateW} and \eqref{eq:fullminupdateH}, making this approach generally computationally expensive.

Fortunately, in practice, one does not need to find the global minima of \eqref{eq:fullminupdateW} and \eqref{eq:fullminupdateH} at each iteration. Instead, using a Block Successive Minimization (BSUM) algorithm~\cite{razaviyayn2013unified}, it is sufficient to minimize approximations of $\mathcal{L}$ which are locally tight upper bounds of $\mathcal{L}$.
To use the BSUM efficiently, these approximation functions need to have three properties: (1) to satisfy the hypotheses of the BSUM Theorem \cite[Theorem 2]{razaviyayn2013unified}, (2) to be as tight as possible, and (3) to be easy to optimize, i.e. to lead to a closed-form solution for each subproblem.

Our contributions can be summarized as follows. We show how regularized Poisson NMF can be efficiently solved using BSUM. We derive tight upper bounds for multiple regularizers and compare our approach with traditional algorithms. We also propose a simple way to introduce linear constraints into the problem and suggest using line search to build even tighter upper bounds. Finally, we propose multiple algorithms for regularized Poisson NMF and conduct numerical simulations to demonstrate the effectiveness of our approach.

\paragraph{Outline of this contribution}

In Section \ref{sec:Related-work}, we provide a review of the literature. In Section \ref{sec:Preliminaries}, we clarify the notation and provide the necessary definitions for the BSUM Theorem \cite[Theorem 2]{razaviyayn2013unified} that will be used for the convergence of our algorithm. In Section \ref{sec:Subproblem-optimisation}, we develop convenient approximations of the objective and regularization functions leading to sub-problems with closed-form solutions. In Section \ref{sec:Generalized-simplex-constraint}, we explore how to modify the optimization scheme to introduce generalized simplex constraints. In Section \ref{sec:Algorithms}, we present our algorithms. Section \ref{sec:application} provides numerical applications of our algorithm, and Section \ref{sec:Conclusion} concludes this contribution.

\section{Related work}
\label{sec:Related-work}

\paragraph{Applications of Poisson Distribution likelihood Maximization}
The maximization of likelihood in a Poisson distribution finds relevance in various applications, prompting the resolution of the problem outlined in Equation \eqref{eq:general-problem}.

Many such applications arise in the domain of physical constrained linear unmixing problems~\cite{kannan2018deep}. Some noteworthy instances encompass:
1. Scanning transmission electron microscopy (STEM) \cite{teurtrie2024stem, shiga2016sparse, cacovich2018unveiling, jany2017retrieving},
2. Hyperspectral Raman and optical imaging \cite{kano2016hyperspectral, wabuyele2005hyperspectral, fu2013hyperspectral},
3. Tensor SVD applied to denoise atomic-resolution 4D scanning transmission electron microscopy \cite{zhang2020denoising},
and 4. Non-local Poisson PCA denoising \cite{salmon2014poisson, yankovich2016non}.
It is noteworthy that many of these applications predominantly employ the L2 case, which offers a comparatively simpler solution. As a result, data is often renormalized to convert Poisson distributions into Gaussian distributions \cite{kotula2003automated}. Nevertheless, the efficacy of these applications could be significantly enhanced by the development of algorithms tailored explicitly for the Poisson case.

Furthermore, Problem \eqref{eq:general-problem} also surfaces in hyperspectral image denoising, where noise is assumed to follow a Poisson distribution \cite{zou2018restoration, ye2014multitask}. In the domain of text mining, the Poisson distribution assumption is frequently utilized for modeling word occurrences based on latent variables such as categories, leading to the problem formulation depicted by \eqref{eq:general-problem} \cite{hofmann1999probabilistic, mei2005discovering}. Additionally, within the context of recommender systems, several matrix factorization problems can be reformulated into the structure of \eqref{eq:general-problem} \cite{singh2008relational, gopalan2013scalable, koren2009matrix}.

\paragraph{Other Optimization Approaches}

The literature offers a limited number of optimization methods suitable for addressing the problem presented by Equation \eqref{eq:general-problem}. This constraint arises from the requirement of many optimization techniques to have a continuously differentiable gradient Lipschitz function. Examples of such techniques include gradient descent \cite{panageas2019first}, perturbed gradient descent \cite{jin2017escape}, nonlinear conjugate gradient method \cite{dai1999nonlinear}, various proximal point minimization algorithms \cite{komodakis2015playing}, and second-order methods like the Newton-CG algorithms \cite{royer2020newton,royer2018complexity}.

One potentially attractive direction is the utilization of Proximal Alternating (Linearized) Minimization (PALM) \cite{bolte2014proximal} or Proximal Alternating Minimization (PAM) \cite{attouch2010proximal}. These algorithms are designed to solve problems of the form:
\[
\minimize_{\MW,\MH}R_W\left(\MW\right)+R_H\left(\MH\right)+\mathcal{L}\left(\MW,\MH\right)
\]
PALM employs a Gauss-Seidel iteration scheme, consisting of the following sub-problems:
\begin{align*}
    \bm{\MW}^{k+1} &  = \argmin_{\MW}  R_W \left(\MW \right)+\mathcal{L}\left(\MW ,\MH^k\right) + c_W \left\| \MW - \MW^k \right\|_2^2  \\
    \bm{\MH}^{k+1} & = \argmin_{\MH}  R_H\left(\MH\right)+\mathcal{L}\left(\MW^{k+1},\MH\right) + c_H \left\| \MH - \MH^k \right\|_2^2
\end{align*} 
Unfortunately, PALM requires the objective function $\mathcal{L}$ to possess a Lipschitz gradient, which is not the case in our scenario. Additionally, the Gauss-Seidel iterations generally lack a closed-form solution, resulting in a slow algorithm with sub-iterations.

To overcome the non-Lipschitz gradient issue, Bregman gradient descent (B-GD) \cite[Algorithm 1.1]{li2019provable} can be considered. This type of algorithm has been extended to alternating minimization \cite[Algorithm 1.3 and 1.4]{li2019provable}. Such an approach can be adapted to our case, as the objective function \eqref{eq:general-problem} exhibits relative smoothness for most regularization scenarios (see \eqref{eq:relative-smoothness}). However, this optimization scheme involves non-tight majorization functions, leading to slow convergence, as discussed in Section~\ref{subsec:Tight-majorizing-functions}.

Given the presence of two blocks of variables, Block Coordinate Descent (BCD) algorithms \cite{tseng2001convergence} naturally emerge as a potential solution. In fact, previous work \cite{kim2014algorithms} demonstrates that many existing approaches can be viewed as Block Coordinate Descent (BCD) problems. However, a primary challenge with BCD lies in its propensity to necessitate full minimization of the sub-problems, which proves to be challenging in the Poisson case. In the L2 case, the subproblems often have closed-form solutions. To address this concern, we explore BSUM algorithms \cite{razaviyayn2013unified} in this study, a generalization of BCD that avoids the requirement for full minimization of the subproblems, instead using upper bounds for the objective function.

\paragraph{Non-Negative Matrix Factorization}
Non-Negative Matrix Factorization (NMF) algorithms have been extensively studied, and for a comprehensive review, one can refer to \cite{gillis2014and}. Initially formulated as Positive Matrix Factorization for the Gaussian (L2-NMF) case by \cite{paatero1997least}, NMF has seen numerous algorithmic developments. In the L2 case, popular approaches include the Alternating Nonnegative Least Squares (ANLS) framework \cite{kim2008nonnegative, kim2011fast, lin2007projected} and the Hierarchical Alternating Least Squares (HALS) method \cite{cichocki2009fast, cichocki2007hierarchical}. As for the Poisson case, known as KL NMF, the first algorithm using Multiplicative Updates (MU) was proposed by \cite{lee1999learning}, with later demonstrations of its convergence provided in \cite{lee2001algorithms} for both Poisson and Gaussian cases. A more rigorous convergence analysis is presented in \cite{lin2007convergence}.

Considering the specific problem of Poisson KL-NMF, there have been a few notable contributions. For instance, \cite{sun2014alternating} employed the Alternating Direction Method of Multipliers (ADMM) with the variable change $X=WH$. The Primal Dual algorithm, based on the framework from Chambolle-Pock, was explored by \cite{yanez2017primal}. Moreover, \cite{hien2021algorithms} conducted a comparative study of various optimization algorithms for KL NMF, including MU \cite{lee2001algorithms}, ADMM \cite{sun2014alternating}, Primal Dual \cite{yanez2017primal}, and Cyclic Coordinate Descent Method \cite{hsieh2011fast}. They also introduced three new algorithms for KL-NMF, namely Block Mirror Descent Method, A Scalar Newton-Type Algorithm, and A Hybrid SN-MU Algorithm.

In the introduction, we mentioned that there are few contributions that address the problem of regularized NMF, with many of them focusing on the L2 case. \cite{kim2014algorithms} demonstrated how many existing works can be cast as Block Coordinate Descent (BCD) problems, allowing the derivation of MU update rules for different regularizers, such as L1 for sparsity. However, their work is limited to L2 NMF. Xu et al.~\cite{xu2013block} proposed a general optimization scheme for block multiconvex optimization using block coordinate descent, which can accommodate regularization on each block. Although applied to L2-NMF, such an approach may lead to algorithms with sub-iterations.
In the context of L2-NMF with sparsity constraints, \cite{teboulle2020novel} presented an approach to address this scenario. Additionally, \cite{taslaman2012framework} introduced a framework for handling L2-NMF with Lipschitz regularizers, akin to our term $s_L$. Other forms of regularization have also been explored, such as graph-based \cite{cai2010graph} or simplex constraint \cite{hofmann1999probabilistic}.

Regarding the specific problem of regularized KL loss, \cite{he2016fast} provided a notable contribution. However, their work focused solely on the subproblems of the NMF problem, rather than addressing the NMF problem itself. Notably, one could potentially employ a similar approach to solve the subproblems of \eqref{eq:general-problem} using BCD. Nonetheless, this would result in a less efficient algorithm with sub-iterations.

\section{Preliminaries\label{sec:Preliminaries}}

\subsection{Notation}


We reserve capital letters for matrices and vectors, e.g., $\Mat{A},\bm{a}$. We use $\bm{a_{i}}$ to refer to the $i$-th numbered vector, and $a_{i}$ to denote the $i$-th element of vector $\bm{a}$. 
The $j$-th element of vector $\bm{a}_{i}$ or the element at the $i$-th row and $j$-th column of matrix $\Mat{A}$ is denoted as $a_{ij}$.

$\Mat{A}\geq0$ and $\bm{a}\geq0$ indicate that all entries of matrix $\Mat{A}$ or vector $\bm{a}$ are greater than or equal to $0$, i.e., $a_{ij}\geq0$ for all $i,j$. 
$\Mat{A}^{\top},\bm{a}^{\top}$ represent the transpose of $\Mat{A}$ and $\bm{a}$, respectively. 
We use $\Mat{A}^{t}$ to denote $\Mat{A}$ at step $t$. 
$\Mat{A}^{t}{}^{\top}$ denotes the transpose of $\Mat{A}^{t}$. 
$\circ$ and $\oslash$ denote elementwise multiplication (also known as the Hadamard product) and division for matrices, respectively. 
For example, $[\Mat{A}\circ\Mat{B}]_{ij}=a_{ij}b_{ij}$ and $[\Mat{A}\oslash\Mat{B}]_{ij}=\frac{a_{ij}}{b_{ij}}$.

As mentioned in the introduction, the matrix to be factorized is generally denoted as $\MY\approx\MW\MH$, where $\MW$ and $\MH$ are its factors. 
We will use $\bm{w},\bm{h}$ as the vectorized versions of $\MW,\MH$, while $\bm{w}_{i},\bm{h}_{i}$ will denote the $i$-th column of $\MW,\MH$. 
$x,\bm{x}$ are general variables that can replace either $\bm{w}$ or $\bm{h}$. 
We use $\mathcal{L}\left(\MW,\MH\right)$ as the general loss function, and $f\left(\bm{x}\right)$ is broadly used to denote a multivariate scalar function.

Finally, we commonly employ calligraphic notation for variable domains. Let $\mathcal{X}$ serve as a generic domain for the variable $x$. In practical terms, it is frequently defined by the $\geq \epsilon$ constraint, specifically as $\mathcal{X}=\left\{\bm{x}\in\Rbb^{m}\vert\bm{x}\geq\epsilon\right\}$. Additionally, we utilize $C_w$ and $C_h$ to represent the domains of $\bm{w}$ and $\bm{h}$, respectively.

\subsection{Definitions}
\label{subsec:Definitions}

In this contribution, we consider the loss function $\mathcal{L}\left(\bm{w},\bm{h}\right)$ with two blocks of variables: $\bm{w}\in\mathcal{C}_{w}$ and $\bm{h}\in\mathcal{C}_{h}$, where $\mathcal{C}_{w}\subset\Rbb^{m_{w}}$ and $\mathcal{C}_{h}\subset\Rbb^{m_{h}}$ are both non-empty convex sets. Here, $\bm{w}$ and $\bm{h}$ correspond to the vectorized versions of the two matrices $\MW$ and $\MH$, respectively. Therefore, $\mathcal{C}_{w}$ and $\mathcal{C}_{h}$ often correspond to the sets $\bm{w}\geq\epsilon$ and $\bm{h}\geq\epsilon$ with $\epsilon>0$. Let us use $\bm{z}=\left[\bm{w},\bm{h}\right]$ to denote all the variables. We have $\bm{z}\in\mathcal{C}=\mathcal{C}_{w}\times\mathcal{C}_{h}\subset\Rbb^{m}=\Rbb^{m_{w}}\times\Rbb^{m_{h}}$, where the total dimension of the problem is $m=m_{w}+m_{h}$.
\begin{definition}[Directional derivative]
Let $\mathcal{L}:\mathcal{C}\rightarrow\Rbb$ be a scalar function, where $\mathcal{C}\subset\Rbb^{m}$ is a convex set. The directional derivative of $\mathcal{L}$ at point $\bm{x}$ in the direction $\bm{d}$ is defined by
\[
\mathcal{L}^{\prime}\left(\bm{z};\bm{d}\right)\coloneqq\lim_{\lambda\downarrow0}\frac{\mathcal{L}\left(\bm{z}+\lambda\bm{d}\right)-\mathcal{L}\left(\bm{z}\right)}{\lambda}
\]
\end{definition}
Note that when $\mathcal{L}$ is differentiable, $\mathcal{L}'(\bm{z};\bm{d}) = \nabla\mathcal{L}(\bm{z})\bm{d}^{\top}$ since $\bm{d}$ and $\bm{z}$ are row vectors. In this contribution, almost all functions are differentiable on the domain of interest.
\begin{definition}[Coordinatewise Minimum]
The point $\bm{z} = [\bm{w},\bm{h}]\in\mathcal{C}$ is a coordinatewise minimum of a function $\mathcal{L}$ if 
\[
\mathcal{L}\left(\bm{w}+\bm{d}_{w},\bm{h}\right)\geq\mathcal{L}\left(\bm{w},\bm{h}\right)\hspace{1em}\forall\bm{d}_{w}\in\mathbb{R}^{m_{w}}\hspace{1em}\text{with}\hspace{1em}\bm{w}+\bm{d}_{w}\in\mathcal{C}_{w}
\]
\[
\mathcal{L}\left(\bm{w},\bm{h}+\bm{d}_{h}\right)\geq\mathcal{L}\left(\bm{w},\bm{h}\right)\hspace{1em}\forall\bm{d}_{h}\in\mathbb{R}^{m_{h}}\hspace{1em}\text{with}\hspace{1em}\bm{h}+\bm{d}_{h}\in\mathcal{C}_{h}
\]
\end{definition}
A coordinatewise minimum is a natural termination point for an alternating minimization algorithm. However, it is important to note that a coordinatewise minimum is not equivalent to a local minimum, as it does not guarantee minimality in all directions. Figure \ref{fig:not_a_local_minimum} (left) provides a counterexample illustrating this.

Another significant concept is the notion of a stationary point, where the gradient is non-negative in all directions.
\begin{definition}[Stationary Points of a function]
Let $\mathcal{L}:\mathcal{C}\rightarrow\Rbb$ be a scalar function, where $\mathcal{C}\subset\Rbb^{m}$ is a convex set. A point $\bm{x}$ is a stationary point of $\mathcal{L}$ if
\[
\mathcal{L}^{\prime}\left(\bm{z};\bm{d}\right)\geq0\hspace{1em}\forall\bm{d}\vert\bm{z}+\bm{d}\in\mathcal{X}
\]
\end{definition}

We emphasize that a stationary point is not equivalent to a \emph{strict} local minimum as there might be directions where the directional derivative equals 0.
For example, in the simple function $f([\bm{w},\bm{h}]) = (wh-2)^2$, the point $[\bm{w},\bm{h}] = [\sqrt{2},\sqrt{2}]$ has a zero derivative in the direction $[1,-1]$, which corresponds to rescaling the solution as $[\alpha,1/\alpha]$.
Even worse, a stationary point is not necessarily a local minimum, even if it is a coordinatewise minimum, as shown in Figure \ref{fig:not_a_local_minimum} (a). Here, in the diagonal directions, the directional derivative equals 0 but the function is concave in this direction. 

When it comes to the Poisson loss function, there is, at least, a continuous set of local minimas corresponding to rescaling the solution.
This is illustrated in Figure \ref{fig:not_a_local_minimum} (b). Note that the introduction of regularization or constraints can lead to strict local minima, as shown in Figure \ref{fig:not_a_local_minimum} (c).

In this contribution, we prove convergence to a coordinatewise minimum that is also a stationary point. To accomplish this, we will consider a class of functions that are regular at their coordinatewise minima.

\begin{definition}[Regularity of a function at a point]
The function $\mathcal{L}:\mathcal{C}\rightarrow\Rbb$ is said to be regular at the point $\bm{z}\in\mathcal{C}$ if $\mathcal{L}'(\bm{z};\bm{d}) \geq 0$ for all $\bm{d} = [\bm{d}_w,\bm{d}_h]\in\Rbb^{m}$ such that $\mathcal{L}'(\bm{z};[\bm{d}_w,\bm{0}]) \geq 0$ and $\mathcal{L}'(\bm{z};[\bm{0},\bm{d}_h]) \geq 0$.
\end{definition}

\begin{restatable}{lemma}{contiuousdifferentiableregular}
\label{lem:contiuous-differentiable-regular}
Continuously differentiable functions are regular at their coordinatewise minimums.
\end{restatable}
Proof is provided in Appendix~\ref{app:proof-lem-continuous-differentiable-regular}. 
Lemma \ref{lem:contiuous-differentiable-regular} plays a crucial role, as it ensures that the coordinate-wise minimum we converge to in Theorem \ref{thm:TBSUM-convergence} is also a stationary point. In this work, we assume the regularizer to be continuously differentiable on the domain.

\begin{figure}[ht!]
    \centering
    \begin{subfigure}{0.3\textwidth}
        \includegraphics[width=\textwidth]{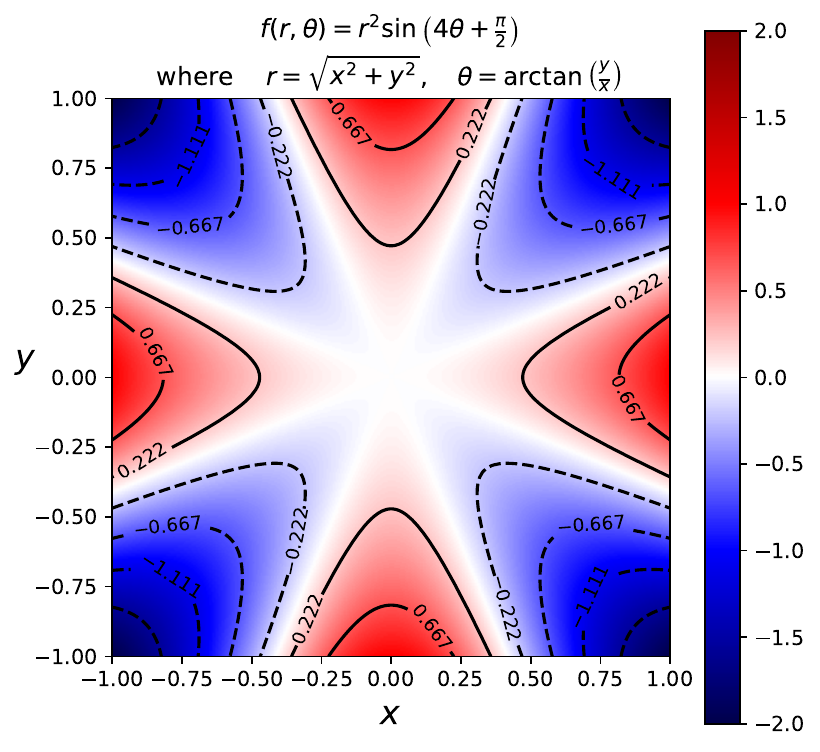}
        \caption{}
    \end{subfigure}
    \hfill
    \begin{subfigure}{0.3\textwidth}
        \includegraphics[width=\textwidth]{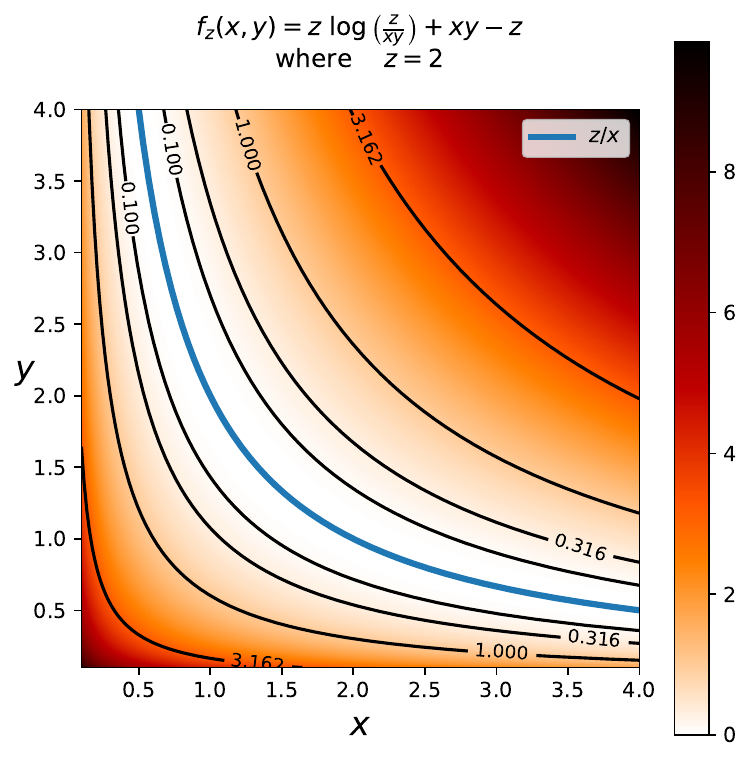}
        \caption{}
    \end{subfigure}
    \hfill
    \begin{subfigure}{0.3\textwidth}
        \includegraphics[width=\textwidth]{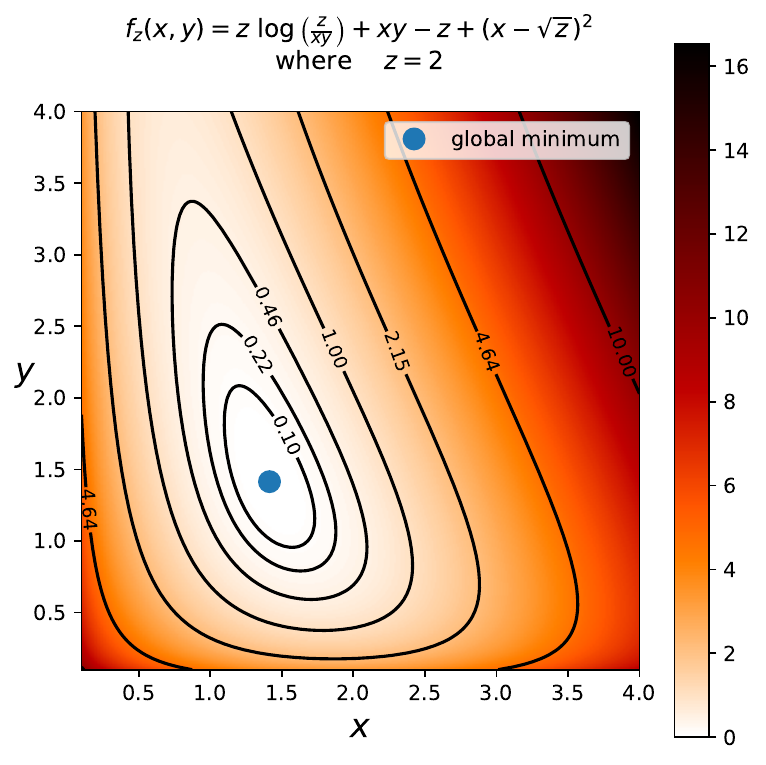}
        \caption{}
    \end{subfigure}


\caption{\label{fig:not_a_local_minimum} (a) A function where the point $(0,0)$ is a coordinatewise minimum along the $x$ and $y$ directions but is not a local minimum. (b) The blue line represents the set of all global minima for the one-dimensional case of \eqref{eq:general-poisson-loss}, where all coordinatewise minima are also stationary points. (c) We show the effect of adding a quadratic regularization term to \eqref{eq:general-poisson-loss}, ensuring the uniqueness of the global minimum.}
\end{figure}

\subsection{Approximation functions}
In order to facilitate optimization algorithms, it is beneficial to work with approximation functions that majorize or approximate the objective function at a given point. One commonly used class of approximation functions is known as \emph{first-order majorization} functions. These functions provide a convenient framework for constructing surrogates and facilitating optimization. We adopt the definition of first-order majorization functions from \cite{razaviyayn2013unified}.
\begin{definition}
\label{def:first-order-majorization-function}\cite[Assumption 1]{razaviyayn2013unified}
A function $g(\bm{x},\bm{x}^{t})$ is said to be a \emph{first-order majorization} of $f$ at the point $\bm{x}^{t}$ if it satisfies the following properties:
\begin{flalign*}
A.1 \hspace{3em} & 
g(\bm{x},\bm{x}^{t}) \geq f(\bm{x}) \hspace{1em} \forall \bm{x},\bm{x}^{t}\in\mathcal{X},\\
A.2 \hspace{3em}&
g(\bm{x}^{t},\bm{x}^{t}) = f(\bm{x}^{t}) \hspace{1em} \forall \bm{x}^{t}\in\mathcal{X},\\
A.3 \hspace{3em}& 
g^{\prime}(\bm{x},\bm{x}^{t};\bm{d})\big\vert_{\bm{x}=\bm{x}^{t}} = f^{\prime}(\bm{x}^{t};\bm{d}) \hspace{1em} \forall \bm{d} \text{ such that } \bm{x}^{t}+\bm{d}\in\mathcal{X},\\
A.4 \hspace{3em}& 
g(\bm{x},\bm{x}^{t}) \text{ is continuous in } (\bm{x},\bm{x}^{t}). 
\end{flalign*}
\end{definition}
It is worth noting that for continuously differentiable functions, the third statement can be equivalently expressed as $\nabla_{\bm{x}}g(\bm{x}^{t},\bm{x}^{t})=\nabla f(\bm{x}^{t})$. Although the definition of first-order majorization functions resembles the concept of surrogate functions introduced in \cite[Definition 2.2]{mairal2015incremental}, the additional requirement for a surrogate function is that $g(\bm{x},\bm{x}^{t})-f(\bm{x})$ is L gradient Lipschitz as defined in \eqref{eq:gradient-lipschitz}. 
Importantly, all majorization functions defined in the following Section~\ref{subsec:Finding-appropriate-surrogate} satisfy this condition and can thus serve as majorization functions.

Conveniently, majorization functions can be built term by term, leveraging their additivity property. This property allows us to combine multiple majorization functions to obtain a new majorization function. 
\begin{lemma}
First-order majorization functions are additive. If $g_{1}(\bm{x},\bm{x}^{t})$ and $g_{2}(\bm{x},\bm{x}^{t})$ majorize $f_{1}(\bm{x})$ and $f_{2}(\bm{x})$ at $\bm{x}^{t}$, respectively, then $g_{1}(\bm{x},\bm{x}^{t})+g_{2}(\bm{x},\bm{x}^{t})$ majorizes $f(\bm{x})=f_1(\bm{x})+f_2(\bm{x})$ at $\bm{x}^{t}$.
\label{lem:First-order-majorization-additivity}
\end{lemma}
\begin{proof}
The additivity property preserves each property of \eqref{def:first-order-majorization-function}.
\end{proof}
Lemma \ref{lem:First-order-majorization-additivity} provides a valuable tool for constructing majorization functions by combining simpler majorization functions. Additionally, when proving that a function is majorizing, it is often unnecessary to explicitly demonstrate the equality of partial derivatives or gradients at $\bm{x}^{t}$. Instead, in the case of differentiable functions, it is typically sufficient to establish the first two properties (A.1 and A.2). According to \cite[Proposition 1]{razaviyayn2013unified}, properties A.3 and A.4 follow as a consequence. Intuitively, one can observe that the continuity of the gradient ensures that the majorization function $g$ shares the tangent spaces with $f$ at the point $\bm{x}^{t}$.

    
\subsection{Two Blocks Successive Minimization (TBSUM)}


The TBSUM algorithm is designed to solve the following problem:
\begin{equation}
\begin{split}
\minimize_{\bm{h},\bm{w}} ~ & \mathcal{L}\left(\bm{w},\bm{h}\right) \\
\text{such that}~~ & \bm{h}\in\mathcal{C}_{h},\bm{w}\in\mathcal{C}_{w}
\end{split}\label{eq:tbsum-opt-problem}
\end{equation}
It relies on two first-order majorizing functions: $g_w(\bm{w},\bm{w}^t,\bm{h}^t)$ and $g_h(\bm{h},\bm{h}^t,\bm{w}^t)$, which majorize $\mathcal{L}(\bm{w},\bm{h})$ at $\left(\bm{w}^t,\bm{h}^t\right)$ for all $\bm{w}^t\in\mathcal{C}_w$ and $\bm{h}^t\in\mathcal{C}_h$. The construction of these functions will be presented in Section \ref{sec:Subproblem-optimisation}. The TBSUM algorithm, outlined in Algorithm \ref{alg:TBSUM}, alternates between minimizing $g_w$ and $g_h$. It is assumed that the subproblem solutions are unique.
Theorem~\ref{thm:TBSUM-convergence} establishes the convergence of the TBSUM algorithm, which is a variant of the algorithm presented in \cite[Theorem 2a]{razaviyayn2013unified} adapted for solving the specific problem at hand.
\begin{algorithm}[h!]
\caption{TBSUM: Two-Block Successive Minimization Algorithm \label{alg:TBSUM}}
\begin{algorithmic}[1]
\Require Initialize the variables to a feasible point $\bm{w}^{0}\in\mathcal{C}_{w}$, $\bm{h}^{0}\in\mathcal{C}_{h}$, and set $t=0$
\Repeat
\State $\bm{w}^{t+1} \gets \argmin_{\bm{w}\in\mathcal{C}_{w}}g_{w}\left(\bm{w},\bm{w}^{t},\bm{h}^{t}\right)$
\State $\bm{h}^{t+1} \gets \argmin_{\bm{h}\in\mathcal{C}_{h}}g_{h}\left(\bm{h},\bm{h}^{t},\bm{w}^{t+1}\right)$
\State $t \gets t+1$
\Until{some convergence criterion is met}
\end{algorithmic}
\end{algorithm}
\begin{theorem}[Convergence of TBSUM Algorithm~\ref{alg:TBSUM}\label{thm:TBSUM-convergence}]
Given two quasi-convex first order majorizing functions $g_{h}\left(\bm{w},\bm{w}^{t},\bm{h}^{t}\right)$
and $g_{w}\left(\bm{h},\bm{h}^{t},\bm{w}^{t+1}\right)$ of $\mathcal{L}\left(\bm{w},\bm{h}\right)$
at $\left(\bm{w}^{t},\bm{h}^{t}\right),\forall\bm{w}^{t},\bm{h}^{t}\in\mathcal{C}_{w}\times\mathcal{C}_{h}$.
Furthermore assuming that the two subproblems in the TBSUM Algorithm
\ref{alg:TBSUM} have unique solutions for any points $\bm{w}^{t}\in C_{w}$,
$\bm{h}^{t}\in C_{h}$. Then, every limit point $\bm{\bm{z}=\left[\bm{w},\bm{h}\right]}$
of the iterates generated by the TBSUM Algorithm \ref{alg:TBSUM}
is a coordinatewise minimum of \eqref{eq:tbsum-opt-problem}. In addition,
if $\mathcal{L}$ is regular at any point $\bm{z}\in\mathcal{C}$,
then $\bm{z}$ is a stationary point of \eqref{eq:tbsum-opt-problem}.
\end{theorem}

\section{Subproblem minimization\label{sec:Subproblem-optimisation}}
In this section, we focus on constructing the appropriate majorization functions $g_{h}\left(\bm{w},\bm{w}^{t},\bm{h}^{t}\right)$ and $g_{w}\left(\bm{h},\bm{h}^{t},\bm{w}^{t+1}\right)$ for our problem \eqref{eq:general-problem}. Since we consider the same type of regularization for $\bm{w}$ and $\bm{h}$, both subfunctions have the same form.

Practically, the loss function can be rewritten as
\[
\mathcal{L}_{\MY}\left(\MW,\MH\right){\color{brown}+R\left(\MW,\MH\right)}=\sum_{j}f_{w}\left(\bm{w}_{j}\right)=\sum_{i}f_{h}\left(\bm{h}_{i}\right)
\]
where $\bm{w}_{j}$ and $\bm{h}_{i}$ are the $i^{th}$ row and $j^{th}$ column of $\MW$ and $\MH$. The functions $f_{w}$ and $f_{h}$ have the form:
\begin{equation}
f\left(\bm{x}\right)={\color{green}-\sum_{i=1}^{m}\left(b_{i}\log\left(\bm{a}_{i}^{\top}\bm{x}\right)+\bm{a}_{i}^{\top}\bm{x}\right)}{\color{blue}+s_{L}\left(\bm{x}\right)}{\color{orange}+s_{R}\left(\bm{x}\right)}{\color{purple}+\sum_{j=1}^{n}s_{C}\left(x_{j}\right)}.\label{eq:function2majorize}
\end{equation}
Therefore, in this section, our objective is to find majorization functions for \eqref{eq:function2majorize}. Once this is done, we will provide closed-form solutions for steps 1 and 2 of Algorithm \ref{alg:TBSUM}. It is worth noting that each term of \eqref{eq:function2majorize} can be handled separately using the additivity property of majorizing functions (Lemma \ref{lem:First-order-majorization-additivity}).

\subsection{Majorizing functions}
\label{subsec:Finding-appropriate-surrogate}
The following four lemmas provide majorizing functions for the different term of our objective function. Proofs are provided in Appendix~\ref{app:proof-majorizing-functions}.

In order to develop an efficient algorithm, our objective is to identify majorizing functions that result in sub-problems with closed-form tractable solutions. Often, this can be accomplished under two conditions: 1. all the majorizing functions are of the same form, and, 2. the majorization function is separable with respect to the variables $\bf{x}$, i.e., $g(\bf{x}) = \sum_i g_i (x_i)$. Within the scope of this contribution, we consider two forms of majorizing functions: quadratic $g(x) = a + bx + cx^2 $ and logarithmic $g(x) = a + bx - c \log(x) $. 


First, we propose a majorization scheme for the logarithmic term $\log\left(\bm{a}_{i}^{\top}\bm{x}\right)$ in the objective function \eqref{eq:function2majorize}. We utilize a widely used majorization technique based on the concavity of the logarithm function. This technique has been employed in the original work by Lee and Seung~\cite{lee2001algorithms} as well as in many EM (Expectation-Maximization) schemes.
\begin{restatable}[Log majorization]{lemma}{logmajorization}
    \label{lem:EM-majorization} Assuming $\bm{a}\circ\bm{x}>0$, for $\bm{x}\in\mathcal{C}$,
    let us define $q_{j}=\frac{a_{j}x_{j}^{t}}{\sum_{k}a_{k}x_{k}^{t}}$
    for $\bm{x}^{t}\in\mathcal{C}$, then $g\left(\bm{x},\bm{x}^{t}\right)=-\sum_{j}q_{j}\log\left(\frac{a_{j}x_{j}}{q_{j}}\right)$
    is a first order majorizing function of $f\left(\bm{x}\right)=-\log\left(\bm{a}^{\top}\bm{x}\right)=-\log\left(\sum_{j}a_{j}x_{j}\right)$.
\end{restatable}

We now proceed to majorize the different terms of the regularisation function $s_{L}\left(\bm{x}\right)$, $s_{R}\left(\bm{x}\right)$, and $s_{C}\left(x_{j}\right)$. We can majorize any Lipschitz function using the following lemma.
\begin{restatable}[Lipschitz-majorization]{lemma}{lipschitzmajorization}
Given $s_{L}\left(\bm{x}\right)$ a gradient Lipschitz function with
constant $\sigma_{L}$ over the domain $\bm{x}\in\mathcal{C}$. The
functions
\begin{align}
g_{1}\left(\bm{x},\bm{x}^{t}\right) & =s_{L}\left(\bm{x}^{t}\right)+\left(\bm{x}-\bm{x}^{t}\right)^{\top}\nabla s_{L}\left(\bm{x}^{t}\right)+\sigma_{L}\|\bm{x}-\bm{x}^{t}\|_{2}^{2}\label{eq:tighter-majorizing-function-smooth-term}
\end{align}
and 
\begin{equation}
g_{2}\left(\bm{x},\bm{x}^{t}\right)=s_{L}\left(\bm{x}^{t}\right)+\left(\bm{x}-\bm{x}^{t}\right)^{\top}\nabla s_{L}\left(\bm{x}^{t}\right)+2\sigma_{L}\left(\max_{j}x_{j}^{t}\right)\left(\sum_{j}x_{j}^{t}\log\left(\frac{x_{j}^{t}}{x_{j}}\right)-x_{j}^{t}+x_{j}\right)\label{eq:looser-majorizing-function-smooth-term}
\end{equation}
are first oder majorizing functions at $\bm{x}^{t}\in\mathcal{C}.$
\end{restatable}
We note that~\eqref{eq:tighter-majorizing-function-smooth-term} (quadratic majorisation) is tighter than~\eqref{eq:looser-majorizing-function-smooth-term} (logarithmic majorisation) . However, the looser majorisation function is needed to obtain a close form solution for the MU (see Section~\ref{sec:Subproblem-optimisation}). 

Next, the term that is relatively smooth can be majorized using the following lemma.
\begin{restatable}[Relative smoothness majorization]{lemma}{relativesmoothnessmajorization}
\label{lem:bregman-majorization}
Assuming $s_{R}\left(\bm{x}\right)$
a $\sigma_{R}$ relatively smooth function with respect to $\kappa\left(\bm{x}\right)=-\bm{1}^{\top}\log\left(\bm{x}\right)$
for $\bm{x}\in\mathcal{C\subset\mathbb{R}}_{+}^{n}.$ Then the function
\begin{equation}
g\left(\bm{x},\bm{x}^{t}\right)=s_{R}\left(\bm{x}^{t}\right)+\left\langle \nabla s_{R}\left(\bm{x}^{t}\right),\bm{x}-\bm{x}^{t}\right\rangle +\sigma_{R}\sum_{i}^{n}\left(\frac{x_{i}}{x_{i}^{t}}-\log\left(\frac{x_{i}}{x_{i}^{t}}\right)-1\right)\label{eq:majorization-relatively-smooth-function}
\end{equation}
is a first order majorizing function of $s_{R}\left(\bm{x}\right)$
for $\bm{x}^{t}\in\mathcal{C}.$
\end{restatable}

\begin{restatable}[Concave majorisation]{lemma}{concavemajorisation}
Given $s\left(x\right)$ a concave function defined on $x\in\mathcal{C}\subset\mathbb{R}$,
it's linear approximation at the point $x^{t}$
\begin{equation}
g\left(x,x^{t}\right)=s\left(x^{t}\right)+\frac{\partial s\left(x^{t}\right)}{\partial x}\left(x-x^{t}\right)\label{eq:concave-bound}
\end{equation}
is a first order majorization function for $x^{t}\in\mathcal{C}.$
\end{restatable}

\subsection{Subproblem updates\label{subsec:Optimimum-majorizing-function}}
Now that we have defined majorizing functions for each term of \eqref{eq:function2majorize}, we can apply the additivity property of Lemma \ref{lem:First-order-majorization-additivity} to obtain a general majorizing function for $f(\bm{x})$:


\begin{align}
g\left(\bm{x},\bm{x}^{t}\right) & ={\color{green}-\sum_{i=1}^{m}\left(b_{i}\sum_{j}q_{ij}\log\left(\frac{a_{ij}x_{j}}{q_{ij}}\right)+\bm{a}_{i}^{\top}\bm{x}\right)}\nonumber \\
 & {\color{blue}+s_{L}\left(\bm{x}^{t}\right)+\left(\bm{x}-\bm{x}^{t}\right)^{\top}\nabla s_{L}\left(\bm{x}^{t}\right)+2\sigma_{L}\left(\max_{j}x_{j}^{t}\right)\left(\sum_{j}x_{j}^{t}\log\left(\frac{x_{j}^{t}}{x_{j}}\right)-x_{j}^{t}+x_{j}\right)}\label{eq:general-surrogate-with-log}\\
 & {\color{orange}+s_{R}\left(\bm{x}^{t}\right)+\left\langle \nabla s_{R}\left(\bm{x}^{t}\right),\bm{x}-\bm{x}^{t}\right\rangle +\sigma_{R}\sum_{j}\left(\frac{x_{j}}{x_{j}^{t}}-\log\left(\frac{x_{j}}{x_{j}^{t}}\right)-1\right)}\nonumber \\
 & {\color{purple}+\sum_{j=1}^{n}s_{C}\left(x_{j}^{t}\right)+\frac{\partial s_{C}\left(x_{j}^{t}\right)}{\partial x_{j}}\left(x_{j}-x_{j}^{t}\right)}\nonumber 
\end{align}
where $q_{ij}=\frac{a_{ij}x_{j}^{t}}{\sum_{k}a_{ik}x_{k}^{t}}.$ 
We use the colors green, blue, orange, and purple to denote and keep track of the dependencies of the different terms in \eqref{eq:function2majorize}.
Finding the local optimum of the majorizing function will provide us with an update for Algorithm \ref{alg:TBSUM}.
\begin{restatable}[Generalized MU for \eqref{eq:function2majorize}]{proposition}{mumarjorization}
    Assuming $\bm{x}^{t}, \bm{x}, \bm{b}, \Mat{A} > 0$, the first-order majorizing function defined in \eqref{eq:general-surrogate-with-log} is strictly convex, and its global minimum $\bm{x}^{t+1}$ is given by
    \begin{equation}
    x_{j}^{t+1} = x_{j}^{t} \frac{\alpha_{j}^{t}}{\beta_{j}^{t}} \label{eq:update-MU}
    \end{equation}
    where
   \begin{align}
\alpha_{j}^{t} & ={\color{green}\sum_{i}b_{i}\frac{a_{ij}}{\sum_{k}a_{ik}x_{k}^{t}}}{\color{blue}+2\left(\max_{i}x_{i}^{t}\right)\sigma_{L}}{\color{orange}+\frac{\sigma_{R}}{x_{j}^{t}}}\hspace{1em}\text{and}\label{eq:alpha-KL}\\
\beta_{j}^{t} & ={\color{green}\sum_{i}a_{ij}}{\color{blue}+\nabla_{x_{j}}s_{L}\left(\bm{x}^{t}\right)+2\left(\max_{i}x_{i}^{t}\right)\sigma_{L}}{\color{orange}+\nabla_{x_{j}}s_{R}\left(\bm{x}^{t}\right)+\frac{\sigma_{R}}{x_{j}^{t}}}{\color{purple}+\frac{\partial s_C\left(x_{j}^{t}\right)}{\partial x}}.\label{eq:beta-KL}
\end{align}
\end{restatable}
The proof is provided in Appendix~\ref{app:proof-subproblem-updates}.

\paragraph{Generalization of the traditional MU Rule}
We note that \eqref{eq:update-MU} serves as a generalization of the original Multiplicative Update (MU) rule presented in \cite{lee2001algorithms}. Removing the regularization terms (blue, orange, and purple) results in precisely the MU rule as outlined in \cite{lee2001algorithms}.

\paragraph{Connection with (Block) Mirror Descent~\cite[Algorithm 1]{hien2021algorithms}}
Another interesting observation is that the majorization of the relatively smooth term is done similarly to a Bregman proximal method algorithm~\cite{hanzely2021accelerated}. Since the objective function $-\sum_{i=1}^{m}\left(b_{i}\log\left(\bm{a}_{i}^{\top}\bm{x}\right)+\bm{a}_{i}^{\top}\bm{x}\right)$ is relatively smooth, one could drop all terms except for $s_{R}$ and optimize using Block Bregman Proximal Gradient (BBPG)~\cite{teboulle2020novel}. This would result in an algorithm very similar to Block Mirror Descent (BMD), which has recently been proposed for solving Poisson NMF~\cite{hien2021algorithms}. Nevertheless, we advice against this this solution as discussed further in Section~\ref{subsec:Tight-majorizing-functions}.

\paragraph{Alternative majorizing function and Quadratic Update (QU)}
As shown experimentally in Section \ref{sec:application} and illustrated in Figure \ref{fig:tight-majorizing-functions}, having a majorization function as tight as possible leads to faster convergence of the algorithm. 
In \eqref{eq:general-loss}, we deliberately choose to use a looser majorizing function for the term $s_{L}$ in order to recover an algorithm with multiplicative update that generalizes the original approach from~\cite{lee2001algorithms}. However, instead of using \eqref{eq:looser-majorizing-function-smooth-term}, one can also use \eqref{eq:tighter-majorizing-function-smooth-term} when constructing the majorizing function:
\begin{align}
g\left(\bm{x},\bm{x}^{t}\right) & ={\color{green}-\sum_{i=1}^{m}\left(b_{i}\sum_{j}q_{ij}\log\left(\frac{a_{ij}x_{j}}{q_{ij}}\right)+\bm{a}_{i}^{\top}\bm{x}\right)}\label{eq:general-surrogate-with-L2}\\
 & {\color{blue}+s_{L}\left(\bm{x}^{t}\right)+\left(\bm{x}-\bm{x}^{t}\right)^{\top}\nabla s_{L}\left(\bm{x}^{t}\right)+\sigma_{L}\left\Vert \bm{x}-\bm{x}^{t}\right\Vert _{2}^{2}}\nonumber \\
 & {\color{orange}+s_{R}\left(\bm{x}^{t}\right)+\left\langle \nabla s_{R}\left(\bm{x}^{t}\right),\bm{x}-\bm{x}^{t}\right\rangle +\sigma_{R}\sum_{j}\left(\frac{x_{j}}{x_{j}^{t}}-\log\left(\frac{x_{j}}{x_{j}^{t}}\right)-1\right)}\nonumber \\
 & {\color{purple}+\sum_{j=1}^{n}s_{C}\left(x_{j}^{t}\right)+\frac{\partial s_{C}\left(x_{j}^{t}\right)}{\partial x_{j}}\left(x_{j}-x_{j}^{t}\right)}\nonumber 
\end{align}
which is also a strictly convex function. 
\begin{restatable}[QU for \eqref{eq:function2majorize}]{proposition}{qumarjorization}
Assuming $\bm{x}^{t}, \bm{x}, \bm{b}, \Mat{A} > 0$, the first-order majorizing function defined in Equation \eqref{eq:general-surrogate-with-L2} is strictly convex, and its global minimum $\bm{x}^{t+1}$ is given by

\begin{equation}
x_{j}^{t+1}=\frac{-\beta_{j}^{t}+\sqrt{\left(\beta_{j}^{t}\right)^{2}+4\alpha\zeta_{j}^{t}}}{2\alpha}\label{eq:update-with-QU}
\end{equation}
where
\begin{equation}
\alpha={\color{blue}2\sigma_{L}}\hspace{1em}\beta_{j}^{t}={\color{green}\sum_{i}a_{ij}}{\color{blue}+\nabla_{x_{j}}s_{L}\left(\bm{x}^{t}\right)-2\sigma_{L}x_{j}^{t}}{\color{orange}+\nabla_{x_{j}}s_{R}\left(\bm{x}^{t}\right)+\frac{\sigma_{R}}{x_{j}^{t}}}{\color{purple}+\frac{\partial s_{C}\left(x_{j}^{t}\right)}{\partial x}}\hspace{1em}\zeta_{j}^{t}={\color{green}\sum_{i}b_{i}\frac{a_{ij}x_{j}^{t}}{\sum_{k}a_{ik}x_{k}^{t}}}{\color{orange}+\sigma_{R}}.\label{eq:alpha-beta-zeta-quadratic}
\end{equation}
\end{restatable}
The proof is provided in Appendix~\ref{app:proof-subproblem-updates}. Both of these propositions lead to the update rule for our MU and QU algorithms detailed in Section~\ref{sec:Algorithms}. We note also that, with the appropriate assumptions, the update rule \ref{eq:update-MU} and \ref{eq:update-with-QU} will preserve positivity of the variable $\bm{x}$. However, since our desire is also to handle extra constraint, we develop in the next section rigorous approach.

\subsection{Generalized simplex constraint }
\label{sec:Generalized-simplex-constraint}


We need to handle two constraints: 1. the linear constraint $\bm{x}\geq\epsilon$, and, 2. the scale constraint $\bm{e}^{\top}\bm{x}=1$, where $\bm{e}\geq0$. 
While the first one is used to keep the variable non-negative, typically with a strictly positive small $\epsilon$, the second one can set the scale of one of the variables ($\MW$ or $\MH$) in the factorization problem. Furthermore, in the case $\bm{e}=\bm{1}$, the simplex constraint is recovered. 
It turns out that the update rules of \eqref{eq:update-MU} and \eqref{eq:update-with-QU} can simply be updated to handle this constraint. 
The actual optimization problem we want to solve becomes:
\[
\minimize_{{\color{violet}\bm{x}\geq\epsilon}}f\left(\bm{x}\right)\hspace{1em}\text{such that}\hspace{1em}{\color{cyan}\bm{e}^{\top}\bm{x}=1}.
\]
where $f$ is given in \eqref{eq:function2majorize}. 

To solve this problem, we used the KKT approach, i.e, we find points that satisfy the KKT (Karush-Kuhn-Tucker) conditions:
\begin{flalign*}
\hspace{2em}& \text{1. Stationarity} & 
\nabla_{\bm{x}}L\left(\dot{\bm{x}},\nu,\bm{\mu}\right)=\bm{0}, & \hspace{3cm}\\
& \text{2. Primal feasibility} 
& \begin{cases}
    {\color{cyan}\bm{e}^{\top}\dot{\bm{x}}-1=0}, \\
    {\color{violet}\dot{\bm{x}}\geq\epsilon}
\end{cases},  &\\
& \text{3. Dual feasibility} 
& {\color{violet}\bm{\mu}\geq\bm{0}}, & \\
& \text{4. Complementary slackness} 
& {\color{violet}\bm{\mu}^{\top}\left(-\bm{x}+\epsilon\bm{1}\right)= 0} ,  &
\end{flalign*}
where the Lagrangian is defined as:
\[
L\left(\bm{x},\nu,\bm{\mu}\right)=f\left(\bm{x}\right){\color{cyan}+\nu\left(\bm{e}^{\top}\bm{x}-1\right)}{\color{violet}+\bm{\mu}^{\top}\left(-\bm{x}+\epsilon\bm{1}\right)}.
\]

We follow the same method as developed in Section \ref{subsec:Optimimum-majorizing-function}, except that we majorize the Lagrangian $L\left(\bm{x},\nu,\bm{\mu}\right)$. The resulting first-order majorizing function is given by:
\[
g^{\prime}\left(\bm{x},\bm{x}^{t},\nu,\bm{\mu}\right)=g\left(\bm{x},\bm{x}^{t}\right){\color{cyan}+\nu\left(\bm{e}^{\top}\bm{x}-1\right)}{\color{violet}+\bm{\mu}^{\top}\left(-\bm{x}+\epsilon\bm{1}\right)}
\]
where $g\left(\bm{x},\bm{x}^{t}\right)$ is given in \eqref{eq:general-surrogate-with-log} or \eqref{eq:general-surrogate-with-L2}. We repeat the development of Section \ref{subsec:Optimimum-majorizing-function} (and the proofs of Appendix \ref{app:proof-subproblem-updates}). We end up with and update that is very similar to \eqref{eq:update-MU} or \eqref{eq:update-with-QU}. In the MU case, we end up with:
\[
x_{j}^{t+1}=x_{j}^{t}\frac{\alpha_{j}}{\beta_{j}{\color{cyan}+\nu e_{j}}{\color{violet}-\mu_{j}}},
\]
where the only final difference consists of two terms in cyan and violet ($\alpha_{j}$ and $\beta_{j}$ remain identical). For the QU, we stick to the same update rule (\ref{eq:update-with-QU}), where only $\beta_{j}^{t}$ is modified:
\[
\beta_{j}^{t\prime}=\beta_{j}^{t}{\color{cyan}+\nu e_{j}}{\color{violet}-\mu_{j}}.
\]
This update rule ensures the first of the KKT conditions (stationarity). We now find $\nu, \bm{\mu}$ such that the second KKT condition holds (primal feasibility). It turns out that $\bm{\mu}$ does not need to be computed explicitly. In the MU case, $\mu_j$ is selected to be large enough such that
\begin{align}
x_{j}^{t+1} & =\max\left(x_{j}^{t}\frac{\alpha_{j}}{\beta_{j}{\color{cyan}+\nu e_{j}}},\epsilon\right)=\frac{x_{j}^{t}\alpha_{j}}{\min\left({\color{violet}\frac{x_{j}^{t}\alpha_{j}}{\epsilon}},\beta_{j}{\color{cyan}+\nu e_{j}}\right)}.\label{eq:MU-rule-with-constraint}
\end{align}
In the QU case, we obtain
\begin{align}
x_{j}^{t+1} & =\max\left(\frac{-\left(\beta_{j}^{t}{\color{cyan}+\nu e_{j}}\right)+\sqrt{\left(\beta_{j}^{t}{\color{cyan}+\nu e_{j}}\right)^{2}+4\alpha\zeta_{j}^{t}}}{2\alpha},\epsilon\right)\label{eq:QU-with-constraint}\\
 & =\frac{-\min\left({\color{violet}\frac{\zeta_{j}^{t}}{\epsilon}-\epsilon\alpha},\beta_{j}^{t}{\color{cyan}+\nu e_{j}}\right)+\sqrt{\left(\min\left({\color{violet}\frac{\zeta_{j}^{t}}{\epsilon}-\epsilon\alpha},\beta_{j}^{t}{\color{cyan}+\nu e_{j}}\right)\right)^{2}+4\alpha\zeta_{j}^{t}}}{2\alpha}\nonumber 
\end{align}
Note that dual feasibility and complementary slackness could be verified, but we leave them out for simplicity. We then need to find the value of $\nu$ such that $\bm{e}^{\top}\bm{x}=1$, which is equivalent to searching for
\begin{align}
h_{1}\left(\nu\right) & =\sum_{j}e_{j}\frac{x_{j}^{t}\alpha_{j}}{\min\left({\color{violet}\frac{x_{j}^{t}\alpha_{j}}{\epsilon}},\beta_{j}^{t}{\color{cyan}+\nu e_{j}}\right)}-1=0\label{eq:dichotomy-algo-surrogate}
\end{align}
Similarly, for the quadratic update of (\ref{eq:update-with-QU}), we search for $\nu$ that satisfies
\begin{equation}
h_{2}\left(\nu\right)=\sum_{j}e_{j}\frac{-\min\left({\color{violet}\frac{\zeta_{j}^{t}}{\epsilon}-\epsilon\alpha},\beta_{j}^{t}{\color{cyan}+\nu e_{j}}\right)+\sqrt{\left(\min\left({\color{violet}\frac{\zeta_{j}^{t}}{\epsilon}-\epsilon\alpha},\beta_{j}^{t}{\color{cyan}+\nu e_{j}}\right)\right)^{2}+4\alpha\zeta_{j}^{t}}}{2\alpha}-1=0.\label{eq:dichotomy-algo-QU}
\end{equation}
There is no closed-form solution for $\nu$; however, the value can be found using a simple dichotomy search. Bounds for starting the dichotomy are computed in Appendix~\ref{app:bound-dichotomy}.

\paragraph{Case $\epsilon=0$}
Most of our reasoning relies on the fact that $x_{i}>0$ and, therefore, on the domain constraint $\epsilon>0$. 
We have found that setting $\epsilon$ to a small non-zero value works well in practice. 
However, our approach can likely be generalized to the case where $\epsilon=0$, following the approach of \cite[Section 4]{lin2007convergence}, which studies the unregularized Poisson NMF case.

\section{Algorithms for Poisson matrix factorisation\label{sec:Algorithms}}

Equipped with the update rules developed in the previous Sections \ref{sec:Subproblem-optimisation} and \ref{sec:Generalized-simplex-constraint}, we are ready to tackle the general problem of this contribution\footnote{Here we show the problem with the linear constraint on $\MH$, however, by symmetry, a similar algorithm can be developed with the constraint on $\MW$.} which consists of minimizing \eqref{eq:general-poisson-loss}:
\begin{align}
\dot{\MW},\dot{\MH}=\argmin_{\MW,\MH} & -\left\langle \MY,\log\left(\MW\MH\right)\right\rangle +\left\langle \mathbf{1},\MW\MH\right\rangle +R_{W}\left(\MW\right)+R_{H}\left(\MH\right)\label{eq:general-problem-detail}\\
\text{such that} & \MW\geq\epsilon,\MH\geq\epsilon,\bm{e}^{\top}\MH=\bm{1}\nonumber 
\end{align}
We first observe that with respect to each variable $\MW, \MH$, the problem is separable by column/row. 
For example, given $\MW$, finding the optimal $\MH$ can be done for each column:
\[
\dot{\bm{h}_{i}}=\argmin_{\bm{h}_{i}}-\left\langle \bm{y}_{i},\log\left(\MW\bm{h}_{i}\right)\right\rangle +\left\langle \mathbf{1},\MW\bm{h}_{i}\right\rangle +r_{H}\left(\bm{h}_{i}\right)\hspace{1em}\text{such that}\hspace{1em}\bm{h}_{i}\geq\epsilon,\bm{e}^{\top}\bm{h}_{i}=1
\]
We therefore apply the TBSUM Algorithm~\ref{alg:TBSUM}, where all lines of $\MW$ and all columns of $\MH$ are updated independently, and obtain the two Algorithms~\ref{alg:KL-gen-alg-with-KL-majorization} and~\ref{alg:KL-gen-alg-with-quadratic-majorization}. We note here that the function $max(\cdot, \epsilon)$ ensures that the solutions are $\geq \epsilon$ (non-negativity).

\begin{algorithm}[H]
\caption{MU Algorithm for Regularized Poisson NMF}
\label{alg:KL-gen-alg-with-KL-majorization}
\begin{algorithmic}[1]
\State Initialize the variables $\MW^{0}\geq\epsilon$, $\MH\geq\epsilon$, $t=0$ such that $\bm{e}^{\top}\MH=\bm{1}$.
\While{some convergence criterion is met}
    \For{each line $\bm{w}_{i}^{t\top}$ of $\MW^{t}$}
        \State Compute $\bm{\alpha}_{i}^{t\top}$ and $\bm{\beta}_{i}^{t\top}$ using (\ref{eq:alpha-KL}) and \eqref{eq:beta-KL}.
        \State Update using the MU rule \eqref{eq:MU-rule-with-constraint}: $w_{ij}^{t+1}\leftarrow\max\left(w_{ij}^{t}\frac{\alpha_{ij}^{t}}{\beta_{ij}^{t}},\epsilon\right)$.
    \EndFor
    \For{each column $\bm{h}_{i}^{t}$ of $\MH^{t}$}
        \State Compute $\bm{\alpha}_{i}^{t}$ and $\bm{\beta}_{i}^{t}$ using \eqref{eq:alpha-KL} and (\ref{eq:beta-KL}).
        \State Find the dual variable $\nu_{i}$ by dichotomy of the function \eqref{eq:dichotomy-algo-surrogate} (set $\nu=0$ if no constraint is present).
        \State Update using the MU \eqref{eq:MU-rule-with-constraint}: $h_{ij}^{t+1}\leftarrow\max\left(h_{ij}^{t}\frac{\alpha_{ij}^{t}}{\beta_{ij}^{t}+\nu_{i}e_{j}},\epsilon\right)$.
    \EndFor
    \State $t\leftarrow t+1$.
\EndWhile
\end{algorithmic}
\end{algorithm}

\begin{algorithm}[H]
\caption{QU Algorithm for Regularized Poisson NMF}
\label{alg:KL-gen-alg-with-quadratic-majorization}
\begin{algorithmic}[1]
\State Initialize the variables $\MW^{0}\geq\epsilon$, $\MH\geq\epsilon$, $t=0$ such that $\bm{e}^{\top}\MH=\bm{1}$.
\While{some convergence criterion is met}
    \For{each line $\bm{w}_{i}^{t\top}$ of $\MW^{t}$}
        \State Compute $\alpha^{t}$, $\bm{\beta}_{i}^{t\top}$, and $\bm{\gamma}_{i}^{t\top}$ using (\ref{eq:alpha-beta-zeta-quadratic}).
        \State Update using the QU rule \eqref{eq:QU-with-constraint}: $w_{ij}^{t+1}\leftarrow\max\left(\frac{-\beta_{ij}^{t}+\sqrt{\left(\beta_{ij}^{t}\right)^{2}+4\alpha^{t}\zeta_{ij}^{t}}}{2\alpha^{t}},\epsilon\right)$.
    \EndFor
    \For{each column $\bm{h}_{i}^{t}$ of $\MH^{t}$}
        \State Compute $\alpha^{t}$, $\bm{\beta}_{i}^{t}$, and $\bm{\gamma}_{i}^{t}$ using \eqref{eq:alpha-beta-zeta-quadratic}.
        \State Find the dual variable $\nu_{i}$ by dichotomy of the function \eqref{eq:dichotomy-algo-QU} (set $\nu=0$ if no constraint is present).
        \State Update using the QU rule \eqref{eq:QU-with-constraint}: $h_{ij}^{t+1}\leftarrow\max\left(\frac{-\left(\beta_{ij}^{t}+\nu_{i}e_{j}\right)+\sqrt{\left(\beta_{ij}^{t}+\nu_{i}e_{j}\right)^{2}+4\alpha^{t}\zeta_{ij}^{t}}}{2\alpha^{t}},\epsilon\right)$.
    \EndFor
    \State $t\leftarrow t+1$.
\EndWhile
\end{algorithmic}
\end{algorithm}

\paragraph{Convergence}
The two update rules in steps 5 and 10 correspond to minimizing first-order strictly convex majorization functions. As a result, we can apply Theorem \ref{thm:TBSUM-convergence} \emph{to guarantee convergence towards a coordinate-wise minimum}. 
It is important to note that \emph{this coordinate-wise minimum is also a stationary point}, given that the objective function remains regular for any point $\MW\in\mathcal{C}_{W},\MH\in\mathcal{C}_{H}$.

\subsection{Algorithm complexity\label{subsec:Algorithm-complexity}}
Let's examine the complexity of both Algorithms \ref{alg:KL-gen-alg-with-KL-majorization} and \ref{alg:KL-gen-alg-with-quadratic-majorization} when considering $\MW\in\mathbb{R}^{n\times k}$ and $\MH\in\mathbb{R}^{k\times m}$. In each iteration, the following complexities are observed: \\
(a) Step 4 has a complexity of $\mathcal{O}\left(nmk\right)$. \\
(b) Step 5 has a complexity of $\mathcal{O}\left(nk\right)$.\\
(c) Step 8 has a complexity of $\mathcal{O}\left(nmk\right)$.\\
(d) Step 9 has a complexity of $\mathcal{O}\left(c_{d}km\right)$, where $c_{d}$ denotes the number of iterations performed by the dichotomy.\\
(e) Step 10 has a complexity of $\mathcal{O}\left(mk\right)$.\\
Thus, the overall complexity per iteration can be expressed as $\mathcal{O}\left(nmk\right)+\mathcal{O}\left(c_{d}km\right)=\mathcal{O}\left(\left(n+c_{d}\right)mk\right)$. This indicates that the computational complexity per iteration is linear with respect to the problem size, i.e., $nm$, multiplied by the number of components, i.e., $k$. 
\paragraph{Impact of the dichotomy}
When $n$ is small, the computational cost of the dichotomy in step 4 becomes dominant. Nevertheless, in general, for larger values of $n$, the impact of the dichotomy becomes negligible.

\subsection{Tight Majorizing Functions \label{subsec:Tight-majorizing-functions}}

While we do not make any theoretical contributions concerning the speed of convergence of the algorithm, we want to emphasize the natural fact that \emph{tighter majorizing functions lead to faster convergence}. Therefore, when evaluating an algorithm, we believe that the analysis of the underlying majorizing function is as insightful as the experimental evaluation. As an example, we could have used Block Mirror Descent to solve (\ref{eq:general-problem}), as was done in \cite[Algorithm 1]{hien2021algorithms}. This algorithm uses a Bregman Difference to create a majorization function for the subproblem. However, this would result in a much looser majorization function, which partly explains the slow convergence of this algorithm observed in \cite{hien2021algorithms}. This difference between majorization functions is exemplified in Figure \ref{fig:tight-majorizing-functions}.

\begin{figure}
\centering{}\includegraphics[width=0.45\linewidth]{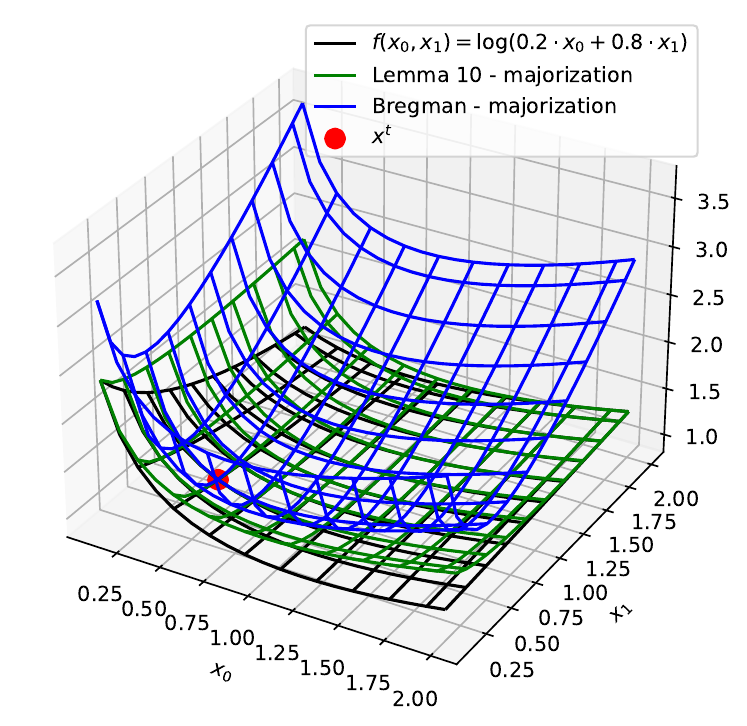}\includegraphics[width=0.45\linewidth]{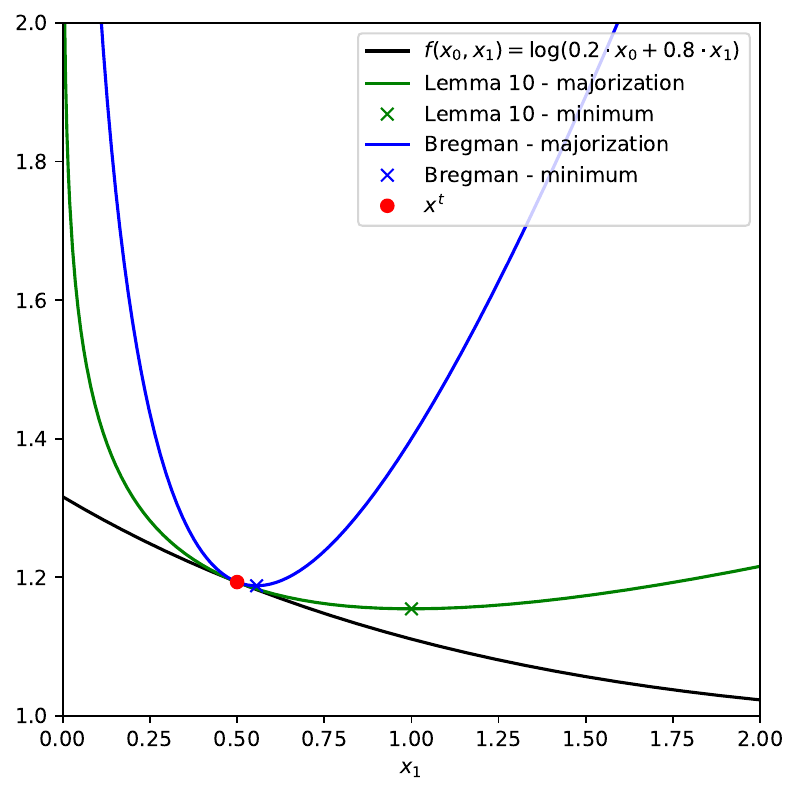}\caption{\label{fig:tight-majorizing-functions}Two first-order majorizing functions for the function $f(x_0,x_1) = \log(0.2x_0 + 0.8x_1)$ are considered. We observe that Lemma (\ref{lem:EM-majorization}) provides a tighter majorization than Lemma (\ref{lem:bregman-majorization}) for $f$, resulting in a more optimal update $\bm{x}^{t+1}$.}
\end{figure}

\paragraph{Linesearch}
By tightening the bounds we used to construct the surrogate function, we can develop a more efficient algorithm. Here, we apply a classic "linesearch" method to the functions $s_L$. However, the same technique can be trivially applied to $s_R$ as well.
First, in (\ref{eq:looser-majorizing-function-smooth-term}) or (\ref{eq:tighter-majorizing-function-smooth-term}), replace the constant $\sigma_L$ with a parameter $\gamma$ and initialize it with $\sigma_L$.
Second, at each iteration, update the parameter $\gamma$ according to the following rule:
\[
\gamma^{t+1}=\begin{cases}
\upsilon\gamma^{t} & \text{if}\hspace{1em}s_{L}\left(\bm{x}\right)\geq g\left(\bm{x},\bm{x}^{t},\gamma\right)\\
\frac{1}{\tau}\gamma^{t} & \text{otherwise.}
\end{cases}
\]
Here, $\upsilon$ and $\tau$ are two update rates that determine how fast $\gamma$ is updated. Choosing values that are too small for these parameters leads to an inefficient linesearch, while selecting values that are too large can result in strong oscillation patterns. Typical values for $\upsilon$ and $\tau$ range from 1.05 to 1.5.
However, it is important to note that when using linesearch, we are not guaranteed to converge, as we might invalidate the assumptions of Theorem \ref{thm:TBSUM-convergence}.

\section{Numerical Simulation \label{sec:application}}

\paragraph{Problem}
In this section, we analyze the speed of convergence of Algorithms \ref{alg:KL-gen-alg-with-KL-majorization} and \ref{alg:KL-gen-alg-with-quadratic-majorization} through numerical simulations. As a regularizer for $\MH$, we consider the Laplacian regularization $R(\MH,\lambda) = \frac{\lambda}{2}\text{tr}(\MH^T\Delta\MH)$, where $\Delta$ represents the two-dimensional Laplacian for the $k$th line of $\MH \in \mathbb{R}^{k \times p^2}$ reshaped as $k$ images of size $p \times p$. Since a straightforward approach to minimize $R(\MH,\lambda)$ is to reduce the amplitude of $\MH$, we add the simplex constraint $\bm{1}^T\MH = \bm{1}$. This leads to the following optimization problem:
\begin{align}
\dot{\MW}, \dot{\MH} = & \arg\min_{\MW, \MH} -\left\langle \MY, \log(\MW\MH) \right\rangle + \left\langle \mathbf{1}, \MW\MH \right\rangle + \frac{\lambda}{2}\text{tr}(\MH^T\Delta\MH) \label{eq:experiment-problem} \
& \text{subject to } \MW \geq \epsilon, \MH \geq \epsilon, \bm{1}^T\MH = \bm{1} \nonumber
\end{align}
This particular problem can be applied in various domains, such as Non-Negative Matrix Factorization for hyperspectral images \cite{lu2012manifold} and remote sensing \cite{li2019cloud} (See Related Work Section \ref{sec:Related-work} for more references and applications).
Our algorithms and regularisations were specifically developed for the \emph{espm} python package \cite{teurtrie2023espm}. All algorithms and experiments can be found in the \emph{espm} package.

\paragraph{Dataset}
We construct two datasets consisting of 50 randomly drawn samples. In the first dataset, both matrices $\MW$ and $\MH$ are randomly generated from a uniform distribution. In the second dataset, each column of $\MW$ corresponds to the sum of Gaussian functions that are randomly centered and scaled. The matrix $\MH$ represents random smooth images. This second dataset is created using the \emph{espm} package \cite{teurtrie2023espm}, where the toy model is used for $\MW$ and $\MW$ is generated using the "laplacian" weight type.
This choice of dataset is selected because it can benefit from the Laplacian regularization on $\MH$.

Once $\MW$ and $\MH$ are generated, the noiseless matrix $\MY$ is obtained as $\MY = \MW\MH$. We introduce noise by independently sampling each element $\tilde{\MY}{ij} \sim \frac{\text{Poisson}(\lambda\MY{ij})}{\lambda}$, where $\lambda$ can be regarded as the noise control parameter. For all samples, we set $\MW \in \mathbb{R}^{n \times k}$ and $\MH \in \mathbb{R}^{k \times p^2}$ with $k = 3$, $p = 64$, and $n$ selected from the set ${25, 100, 500, 1000}$. Thus, the images in the dataset have dimensions of $64 \times 64$.

\paragraph{Results}
We compare the performance of Algorithm \ref{alg:KL-gen-alg-with-KL-majorization} (MU), Algorithm \ref{alg:KL-gen-alg-with-quadratic-majorization} (QU), Block Mirror Descent (similar to \cite[Algorithm 1]{hien2021algorithms}) and the projected gradient algorithm applied to (\ref{alg:KL-gen-alg-with-quadratic-majorization}). 
Figure \ref{fig:convergence-curves} displays the convergence curves for 1000 iterations and $n=25$. 
Although the overall complexity of all algorithms is the same, the time per iteration differs due to the different operations performed within each iteration and the time spent on dichotomy to compute the dual variable $\nu$. 
Therefore, we provide the time in seconds for each algorithm in Figure \ref{fig:convergence-time} for various value of $n$. All results are averaged over 50 repetitions.

\begin{figure}[ht]
\begin{centering}
\begin{tabular}{c>{\centering\arraybackslash}m{0.45\linewidth}>{\centering\arraybackslash}m{0.45\linewidth}}
 & Smooth images samples & Random uniform $\MW$, $\MH$ samples \tabularnewline 
\rotatebox{90}{\makebox[0pt]{Noiseless}} 
& \includegraphics[width=1\linewidth]{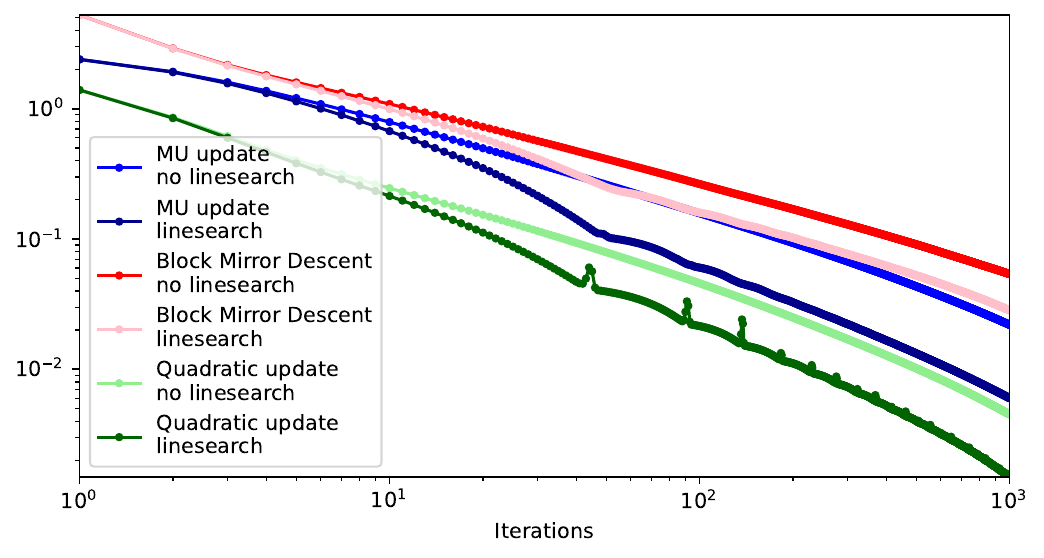} & \includegraphics[width=1\linewidth]{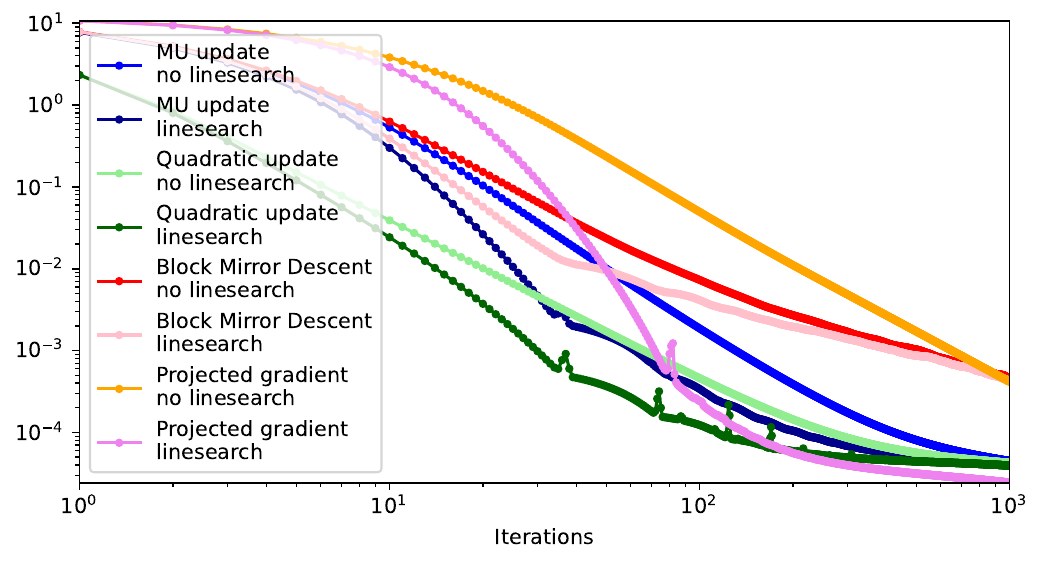}\tabularnewline
\rotatebox{90}{Noisy} & \includegraphics[width=1\linewidth]{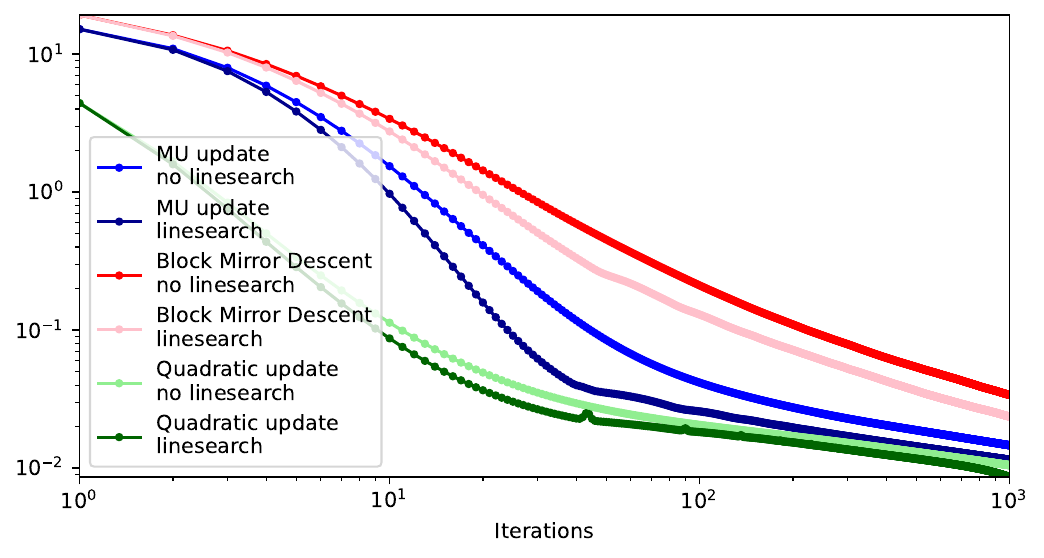} & \includegraphics[width=1\linewidth]{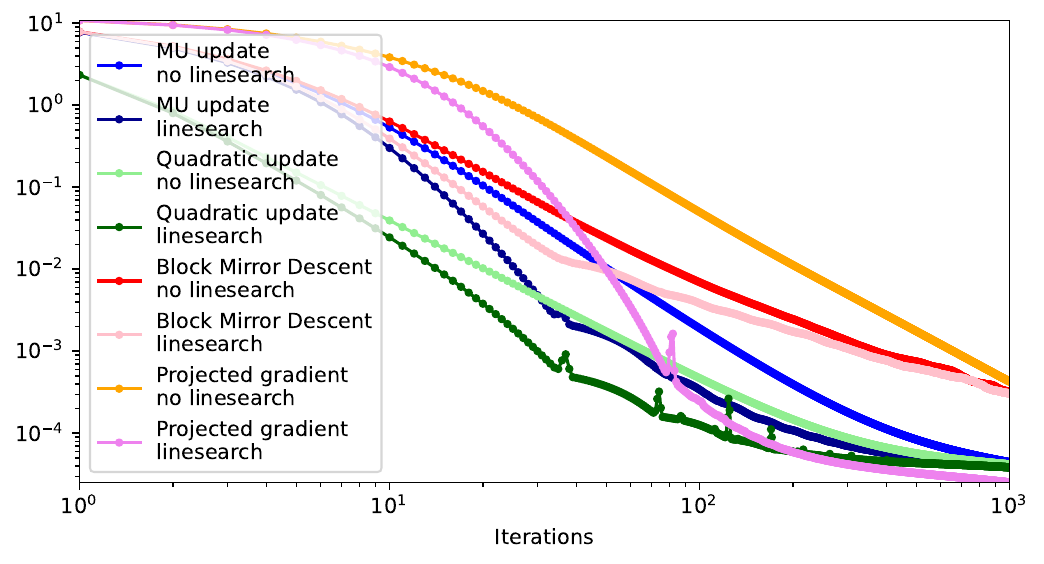}\tabularnewline
\end{tabular}
\par\end{centering}
\caption{\label{fig:convergence-curves}Convergence curves for 1000 iterations.
We remove the minimum loss $\mathcal{L}\left(\dot{\MW},\dot{\MH}\right)$}
\end{figure}

\begin{figure}[ht]
\begin{centering}
\includegraphics[width=0.95\linewidth]{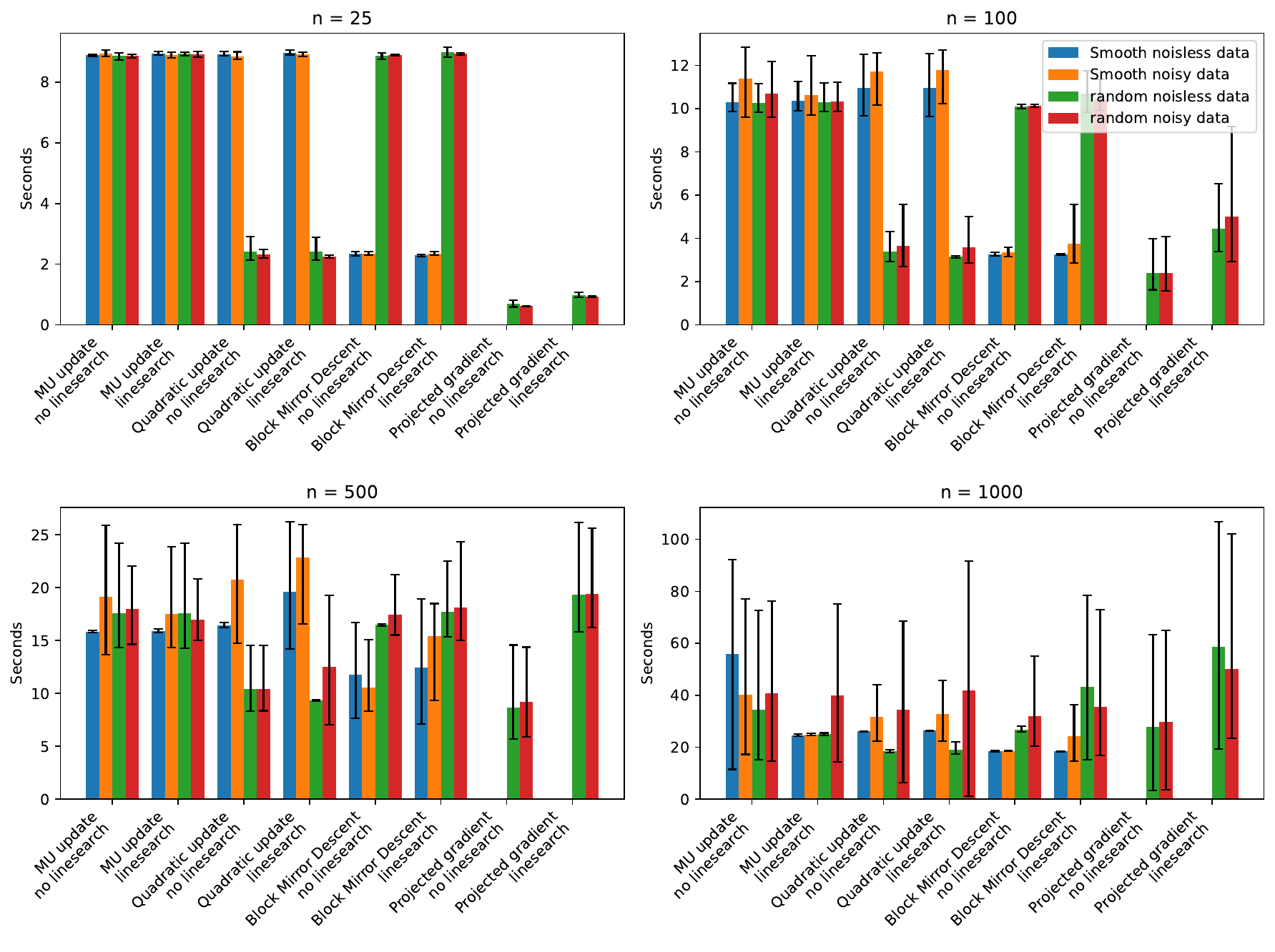}
\end{centering}
\caption{\label{fig:convergence-time}Execution time for 100 iterations for
different problem sizes. Here we fix the dimension of $\MH$ to $3\times64^{2}$
and varies $\MW$ from $25\times3$ to $1000\times3$.}
\end{figure}

\paragraph{Discussion}
Let's discuss the results in more detail: \\
\emph{ - Number of iterations:} Figure \ref{fig:convergence-curves} illustrates the convergence behavior of QU and MU algorithms. It is evident that QU converges faster per iteration compared to MU, which aligns with our expectations due to the tighter majorizing function used in QU. However, it is important to note that the introduction of the linesearch technique, while accelerating convergence, can lead to occasional instability, as indicated by occasional increases in the loss function. This observation supports our earlier discussion in Section \ref{subsec:Tight-majorizing-functions}. 
The challenge with Projected Gradient Descent is that we need to find an initial learning rate that is not too large, as the algorithm can diverge and not too small as the algorithm can be slow. Overall, since the selected learning rate cannot be selected optimally, the algorithm is slower than QU and MU.
The Block Mirror Descent algorithm is also slower than QU and MU, which is consistent with the results of \cite{hien2021algorithms} and can be explained by the fact that the majorizing function used in the algorithm is looser.
\\
\emph{ - Time per iteration:} Figure \ref{fig:convergence-time} presents the total time taken by the algorithms to complete 100 iterations. The results demonstrate that for small values of $n$, the computation of the $p^2$ dual variable during the dichotomy process dominates the overall execution time. However, as $n$ increases, the time spent on dichotomy becomes negligible in comparison. These findings align with the complexity per iteration discussed in Section \ref{subsec:Algorithm-complexity}.

\section{Conclusion\label{sec:Conclusion}}
This contribution is the first to address the Poisson NMF problem with general regularization terms, such as Lipschitz functions, relatively smooth functions, or those expressed as linear constraints. We introduce two new algorithms and demonstrate their convergence to a coordinate-wise minimum, which is also a stationary point. Emphasizing the impact of the majorizing function choice on convergence speed, we validate our findings through numerical simulations. In essence, we believe that this work serves as a helpful guide for developing efficient algorithms suited for regularized Poisson NMF problems.

\appendix

\section{Proofs}
In this Appendix, we provide the different proofs used in the paper.

\subsection{Proof of Lemma~\ref{lem:contiuous-differentiable-regular}}
\label{app:proof-lem-continuous-differentiable-regular}
We note that this lemma and its proof likely exist in the literature, but we were unable to find a reference.
\contiuousdifferentiableregular*
\begin{proof}
If $\mathcal{L}$ is continuously differentiable, the directional derivative can be written as $\mathcal{L}'(\bm{z};\bm{d})=\bm{d}^\top \nabla\mathcal{L}(\bm{z})$. At a coordinatewise minimum $\bm{z}$, we have by definition:
\[
\mathcal{L}^{\prime}\left(\bm{z};\left[\bm{d}_{w},\bm{0}\right]\right)=\nabla\mathcal{L}\left(\bm{z}\right)\left[\bm{d}_{w},\bm{0}\right]^{\top}\geq0
\]
and
\[
\mathcal{L}^{\prime}\left(\bm{z};\left[\bm{0},\bm{d}_{h}\right]\right)=\nabla\mathcal{L}\left(\bm{z}\right)\left[\bm{0},\bm{d}_{h}\right]^{\top}\geq0.
\]
Therefore,
\begin{align*}
\mathcal{L}'(\bm{z};\bm{d}) &= \nabla\mathcal{L}(\bm{z})[\bm{d}_{w},\bm{d}_{h}]^\top \\
&= \nabla\mathcal{L}(\bm{z})([\bm{d}_{w},\bm{0}]^\top + [\bm{0},\bm{d}_{h}]^\top) \\
&= \nabla\mathcal{L}(\bm{z})[\bm{d}_{w},\bm{0}]^\top + \nabla\mathcal{L}(\bm{z})[\bm{0},\bm{d}_{h}]^\top \\
&\geq 0.
\end{align*}
\end{proof}

\subsection{Proof of majorizing functions}
\label{app:proof-majorizing-functions}

\logmajorization*

\begin{proof}
First let us observe that 
\begin{align*}
g\left(\bm{x}^{t},\bm{x}^{t}\right) 
& =-\sum_{j}\frac{a_{j}x_{j}^{t}}{\sum_{k}a_{k}x_{k}^{t}}\log\left(\frac{a_{j}x_{j}^{t}}{\frac{a_{j}x_{j}^{t}}{\sum_{k}a_{k}x_{k}^{t}}}\right) \\
& =-\sum_{j}\frac{a_{j}x_{j}^{t}}{\sum_{k}a_{k}x_{k}^{t}}\log\left(\sum_{k}a_{k}x_{k}^{t}\right) \\
& =-\log\left(\sum_{k}a_{k}x_{k}^{t}\right)=f\left(\bm{x}^{t}\right).    
\end{align*}
Therefore, we have $g\left(\bm{x},\bm{x}^{t}\right)\geq f\left(\bm{x}\right)$. The inequality follows from the convexity of the $-\log$ function:
\begin{align*}
-\log\left(\sum_{j}a_{j}x_{j}\right) & =-\log\left(\sum_{j}q_{j}\frac{a_{j}x_{j}}{u_{j}}\right)\leq-\sum_{j}q_{j}\log\left(\frac{a_{j}x_{j}}{u_{j}}\right)
\end{align*}
where we set $q_{j}=\frac{a_{j}x_{j}^{t}}{\sum_{k}a_{k}x_{k}^{t}}$.
Finally, by continuity, we obtain the $3^{\text{rd}}$ and $4^{\text{th}}$ properties of majorizing functions.
\end{proof}

We now proceed to majorize $s_{L}\left(\bm{x}\right)$, $s_{R}\left(\bm{x}\right)$, and $s_{C}\left(x_{j}\right)$.
\lipschitzmajorization*
\begin{proof}
First it can be trivially observed that 
\[
s_{L}\left(\bm{x}^{t}\right)=g_{1}\left(\bm{x}^{t},\bm{x}^{t}\right)=g_{2}\left(\bm{x}^{t},\bm{x}^{t}\right),
\]
which satisfies the first property. We then take $1^{\text{st}}$
order Taylor expension of $s_{L}$ around $\bm{x}^{t}$ and find 

\[
s_{L}\left(\bm{x}\right)=s_{L}\left(\bm{x}^{t}\right)+\left(\bm{x}-\bm{x}^{t}\right)^{\top}\nabla s_{L}\left(\bm{x}^{t}\right)+\text{\ensuremath{\mathcal{R}}\ensuremath{\left(\bm{x},\bm{x}^{t}\right)}}
\]
where $\text{\ensuremath{\mathcal{R}}\ensuremath{\left(\bm{x},\bm{x}^{t}\right)}}\leq\sigma_{L}\|\bm{x}-\bm{x}^{t}\|_{2}^{2}$
since the function $s_{L}$ is gradient Lipschitz with constant $\sigma_{L}$.
Therefor we have
\begin{align}
s_{L}\left(\bm{x}\right) & \leq s_{L}\left(\bm{x}^{t}\right)+\left(\bm{x}-\bm{x}^{t}\right)^{\top}\nabla s_{L}\left(\bm{x}^{t}\right)+\sigma_{L}\|\bm{x}-\bm{x}^{t}\|_{2}^{2}=g_{1}\left(\bm{x},\bm{x}^{t}\right) \nonumber\\
    & \leq s_{L}\left(\bm{x}^{t}\right)+\left(\bm{x}-\bm{x}^{t}\right)^{\top}\nabla s_{L}\left(\bm{x}^{t}\right)+2\sigma_{L}\left(\max_{j}x_{j}^{t}\right)\left(\sum_{j}x_{j}^{t}\log\left(\frac{x_{j}^{t}}{x_{j}}\right)-x_{j}^{t}+x_{j}\right) \label{eq:ineq-to-be-proved}\\
    & =g_{2}\left(\bm{x},\bm{x}^{t}\right), \nonumber
\end{align}
where \eqref{eq:ineq-to-be-proved} will be shown later in this proof.
By continuity, we obtain the $3^{\text{rd}}$ and $4^{\text{th}}$
property of majorizing functions. We now need to prove \eqref{eq:ineq-to-be-proved} and reformulate it as 


\begin{equation}
\|\bm{x}-\bm{x}^{t}\|_{2}^{2}\leq2\left(\max_{j}x_{j}^{t}\right)\left(\sum_{j}x_{j}^{t}\log\left(\frac{x_{j}^{t}}{x_{j}}\right)-x_{j}^{t}+x_{j}\right)=2\left(\max_{j}x_{j}^{t}\right)D_{GKL}\left(\bm{x}\|\bm{x}^{t}\right).\label{eq:bound_l2_dgkl}
\end{equation}
For simplicity, let us define the function
\[
q\left(\bm{x}\right)=\sum_{j}x_{j}\log\left(x_{j}\right),
\]
with the gradient $\nabla_{x_{j}}q\left(\bm{x}\right)=\log\left(x_{j}\right)+1$
and the Hessian 
\[
\Mat{H}_{ij}^{q}\left(\bm{x}\right)=\frac{\partial q}{\partial x_{i}\partial x_{j}}\left(\bm{x}\right)=\begin{cases}
\frac{1}{x_{j}} & \text{if }i=j\\
0 & \text{otherwise.}
\end{cases}
\]
Note that $q$ is a strictly convex function for $\bm{x}>0$.
We expand the generalized KL divergence:
\begin{align}
D_{GKL}\left(\bm{x}\|\bm{x}^{t}\right) & =\sum_{j}x_{j}^{t}\log\left(x_{j}^{t}\right)-\sum_{j}x_{j}^{t}\log\left(x_{j}\right)-\sum_{j}x_{j}^{t}+\sum_{j}x_{j}\nonumber \\
    & =\sum_{j}x_{j}^{t}\log\left(x_{j}^{t}\right)-\sum_{j}x_{j}\log\left(x_{j}\right)-\sum_{j}\left(\log\left(x_{j}\right)+1\right)\left(x_{j}^{t}-x_{j}\right)\nonumber \\
    & =q\left(\bm{x}^{t}\right)-\left(q\left(\bm{x}\right)+\nabla q\left(\bm{x}\right)^{\top}\left(\bm{x}^{t}-\bm{x}\right)\right)\nonumber \\
    & =\frac{1}{2}\left(\bm{x}^{t}-\bm{x}\right)^{T}\Mat{H}^{q}\left(\tilde{\bm{x}}\right)\left(\bm{x}^{t}-\bm{x}\right),\label{eq:bound_dgkl_lemma}
\end{align}
where $\tilde{\bm{x}}$ is selected such that the last equality holds.
Since $q$ is a strictly convex function, we know that $\tilde{\bm{x}}=\rho\bm{x}+\left(1-\rho\right)\bm{x}^{t}$
for some give $\rho\in[0,1].$ Now we bound the Hessian as
\[
\Mat{H}^{q}\left(\bm{x}\right)\ge\frac{1}{\max_{j}x_{j}}\Mat{I}
\]
and introducing this inquality in \eqref{eq:bound_dgkl_lemma}, we obtain
\[
D_{GKL}\left(\bm{x}\|\bm{x}^{t}\right)\geq\frac{1}{2\max_{j}x_{j}}\|\bm{x}-\bm{x}^{t}\|_{2}^{2},
\]
which is equivalent to \eqref{eq:bound_l2_dgkl} and completes the proof.
\end{proof}

\relativesmoothnessmajorization*
\begin{proof}
The first property $g\left(\bm{x}^{t},\bm{x}^{t}\right)=s_{R}\left(\bm{x}^{t}\right)$
can be trivially verified. Then using by the definition of relatively
smoot function
\[
s_{R}\left(\bm{x}\right)\leq s_{R}\left(\bm{x}^{t}\right)+\left\langle \nabla s_{R}\left(\bm{x}^{t}\right),\bm{x}-\bm{x}^{t}\right\rangle +\sigma_{R}\mathcal{B}_{\kappa}\left(\bm{x},\bm{x}^{t}\right)
\]
where
\[
\mathcal{B}_{\kappa}\left(\bm{x},\bm{x}^{t}\right):=\kappa\left(\bm{x}\right)-\kappa\left(\bm{x}^{t}\right)-\left\langle \nabla\kappa\left(\bm{x}^{t}\right),\bm{x}-\bm{x}^{t}\right\rangle .
\]
Given that $\kappa\left(\bm{x}\right)=-\bm{1}^{\top}\log\left(\bm{x}\right)$,
we compute
\[
\mathcal{B}_{\kappa}\left(\bm{x},\bm{x}^{t}\right)=\sum_{i}^{n}\left(\frac{x_{i}}{x_{i}^{t}}-\log\left(\frac{x_{i}}{x_{i}^{t}}\right)-1\right)
\]
Finally, by continuity, we obtain the $3^{\text{rd}}$ and $4^{\text{th}}$
property of majorizing functions.
\end{proof}

\concavemajorisation*
\begin{proof}
One can simply observe $s\left(x^{t}\right)=g\left(x^{t},x^{t}\right)$.
Then by concavity, we have
\[
s\left(x_{i}\right)\leq s\left(x_{i}^{t}\right)+\frac{\partial s\left(x_{i}^{t}\right)}{\partial x_{i}}\left(x_{i}-x_{i}^{t}\right)
\]
Finally, by continuity, we obtain the $3^{\text{rd}}$ and $4^{\text{th}}$
property of majorizing functions.
\end{proof}

\subsection{Proof of subproblem updates}
\label{app:proof-subproblem-updates}
In this subsection, we present the proof of the subproblem updates used in  the MU and QU algorithms. Let us start with the MU updates. 
\mumarjorization*
\begin{proof}
    Assuming $a_{ij}, b_{i} > 0$, \eqref{eq:function2majorize} is strictly convex, as the green term is strictly convex, and the remaining terms are convex. Consequently, \eqref{eq:function2majorize} possesses a global minimum. To identify this minimum, we seek the stationary point $\nabla_{\bm{x}}g=\bm{0}$. Due to our meticulous selection of majorizing functions, this subproblem becomes separable. When computing the gradient with respect to the variable $x_{j}$, we obtain:
    \begin{align}
    \nabla_{x_{j}}g\left(\bm{x},\bm{x}^{t}\right) & ={\color{green}-\frac{1}{x_{j}}\sum_{i}b_{i}\frac{a_{ij}x_{j}^{t}}{\sum_{k}a_{ik}x_{k}^{t}}+\sum_{i}a_{ij}}\nonumber \\
     & {\color{blue}+\nabla_{x_{j}}s_{L}\left(\bm{x}^{t}\right)+2\left(\max_{i}x_{i}^{t}\right)\sigma_{L}\left(1-\frac{x_{j}^{t}}{x_{j}}\right)}\nonumber \\
     & {\color{orange}+\nabla_{x_{j}}s_{R}\left(\bm{x}^{t}\right)+\sigma_{R}\left(\frac{1}{x_{j}^{t}}-\frac{1}{x_{j}}\right)}\nonumber \\
     & {\color{purple}+\frac{\partial s_{C}\left(x_{j}^{t}\right)}{\partial x}}=0\label{eq:grad_surrogate}
    \end{align}
    where we assume that $x_{j}, x_{j}^{t} > 0$ since $0 \notin \mathcal{C}$. Transforming the above expression, we find a multiplicative update rule \eqref{eq:update-MU} for $x_{j}$.
    \end{proof}
    
The proof of the QU updates is similar to the MU updates, expect that we need to solve a quadratic equation to obtain a closed-form solution.
\qumarjorization*
\begin{proof}
    Assuming $a_{ij}, b_{i} > 0$, Equation \eqref{eq:general-surrogate-with-L2} is strictly convex. This convexity arises from the strict convexity of the green term, coupled with the convexity of the other terms. Consequently, \eqref{eq:general-surrogate-with-L2} possesses a global minimum. To identify this minimum, we seek the stationary point by computing $\nabla_{\bm{x}}g=\bm{0}$:
    \begin{align}
    \nabla_{x_{j}}g\left(\bm{x},\bm{x}^{t}\right) & ={\color{green}-\frac{1}{x_{j}}\sum_{i}b_{i}\frac{a_{ij}x_{j}^{t}}{\sum_{k}a_{ik}x_{k}^{t}}+\sum_{i}a_{ij}}\nonumber \\
     & {\color{blue}+\nabla_{x_{j}}s_{L}\left(\bm{x}^{t}\right)+2\sigma_{L}\left(x_{j}-x_{j}^{t}\right)}\nonumber \\
     & {\color{orange}+\nabla_{x_{j}}s_{R}\left(\bm{x}^{t}\right)+\sigma_{R}\left(\frac{1}{x_{j}^{t}}-\frac{1}{x_{j}}\right)}\nonumber \\
     & {\color{purple}+\frac{\partial s_{C}\left(x_{j}^{t}\right)}{\partial x}=0}\label{eq:grad_majorizing_tight}
    \end{align}
    We observe that it is a separable quadratic function, hence the update rule named Quadratic Update (QU). 
    Solving (\ref{eq:grad_majorizing_tight}) for $x_{j} > 0$ can be rewritten as
    \[
    \alpha x_{j}^{2} + \beta_{j}^{t}x_{j} - \zeta_{j}^{t} = 0,
    \]
    where $\alpha$, $\beta_{j}^{t}$, and $\zeta_{j}^{t}$ are given in (\ref{eq:alpha-beta-zeta-quadratic}). Assuming $\zeta_{j}^{t} \neq 0$, we have $\zeta_{j}^{t} > 0$ and $4\alpha\zeta_{j}^{t} > 0$. 
    Therefore, the previous quadratic equation has two real solutions. 
    Since $\sqrt{\left(\beta_{j}^{t}\right)^{2} + 4\alpha\zeta_{j}^{t}} > \beta_{j}^{t}$, they are of opposite sign. Due to the constraint $x_{j} \geq \epsilon > 0$, we select the positive one, leading to the update rule of \eqref{eq:update-with-QU} for $x_{j}$.
    \end{proof}

\section{Computation of Lower and Upper Bounds for Dichotomy}
\label{app:bound-dichotomy}
In Section \ref{sec:Generalized-simplex-constraint}, we introduced modifications to the MU and QU algorithms to incorporate the positivity constraint $\bm{x} \geq \epsilon$ and the linear constraint $\bm{e}^\top \bm{x} = 1$. 
However, solving for the dual parameter $\nu$ in equations (\ref{eq:dichotomy-algo-surrogate}) for MU or (\ref{eq:dichotomy-algo-QU}) for QU is intractable. To address this, we propose using the dichotomy method to solve for $h(\nu) = 0$. 
Therefore, this appendix provides the computation of lower bound $\nu_{\text{low}}$ and upper bound $\nu_{\text{up}}$ such that $h(\nu_{\text{low}}) < 0$ and $h(\nu_{\text{up}}) > 0$. 
These bounds will serve as convenient initializations for the dichotomy algorithm.

\subsection{Case 1: MU}

For MU, we aim to solve equation \eqref{eq:dichotomy-algo-surrogate} for $\nu$:
\[
h_{1}\left(\nu\right)=\sum_{j}e_{j}\frac{x_{j}^{t}\alpha_{j}}{\min\left(\frac{x_{j}^{t}\alpha_{j}}{\epsilon},\beta_{j}^{t}+\nu e_{j}\right)}-1=0
\]
Terms where $e_j = 0$ can be ignored since they do not contribute to the sum. Assuming $\epsilon > 0$, the function $h_1$ is well-defined for $\nu \in \mathbb{R}$. 
Since $x_j^t \alpha_j > 0$, we have $h_{1}\left(\nu\right)=\frac{\|\bm{e}\|_{1}}{\epsilon}-1$ for $\nu\leq\nu_{\lim}=\min_{j}\left(\frac{\frac{x_{j}^{t}\alpha_{j}}{\epsilon}-\beta_{j}^{t}}{e_{j}}\right)$, and $h_1$ is monotonically decreasing for $\nu\in\left[\nu_{\lim},\infty\right[$. 
Assuming $\frac{\|\bm{e}\|_{1}}{\epsilon}-1\geq0$ (which ensures the feasibility of the constraints $\bm{x} \geq \epsilon$ and $\bm{e}^\top \bm{x} = 1$), we have $h_1(\nu_{\text{lim}}) \geq 0$ and $\lim_{\nu \to \infty} h_1(\nu) = -1$. Thus, there exists exactly one root for the function $h_1$.

\paragraph{Negative bound}
First, let's find $\nu_{\text{low}}$ such that $h_1(\nu) < 0$ for $\nu_{\text{low}} \leq \nu < \infty$. We can bound $h_1$ as follows:
\begin{align*}
h_{1}\left(\nu\right) & =\sum_{j}\frac{x_{j}^{t}\alpha_{j}}{\min\left(\frac{x_{j}^{t}\alpha_{j}}{\epsilon e_{j}},\frac{\beta_{j}^{t}}{e_{j}}+\nu\right)}-1\\
 & \leq n\frac{\max_{j}x_{j}^{t}\alpha_{j}}{\min_{j}\min\left(\frac{x_{j}^{t}\alpha_{j}}{\epsilon e_{j}},\frac{\beta_{j}^{t}}{e_{j}}+\nu\right)}-1<0
\end{align*}
Therefore one possible bound is
\[
\nu_{\text{low}}=n\max_{j}x_{j}^{t}\alpha_{j}-\min_{j}\frac{\beta_{j}^{t}}{e_{j}}.
\]

\paragraph{Positive bound}
Similarly, let's find $\nu_{\text{up}}$ such that $h_1(\nu) \geq 0$ for $\nu_{\text{up}} \geq \nu \geq \nu_{\text{lim}}$. Note that $\nu_{\text{lim}}$ is not a good bound when $\epsilon$ is small. We can bound $h_1$ as follows:
\begin{align*}
h_{1}\left(\nu\right) & =\sum_{j=1}^{n}\frac{x_{j}^{t}\alpha_{j}}{\min\left(\frac{x_{j}^{t}\alpha_{j}}{\epsilon e_{j}},\frac{\beta_{j}^{t}}{e_{j}}+\nu\right)}-1\\
 & \geq\max_{j}\frac{x_{j}^{t}\alpha_{j}}{\min\left(\frac{x_{j}^{t}\alpha_{j}}{\epsilon e_{j}},\frac{\beta_{j}^{t}}{e_{j}}+\nu\right)}-1\\
 & \geq\max_{j}\frac{x_{j}^{t}\alpha_{j}}{\frac{\beta_{j}^{t}}{e_{j}}+\nu}-1>0,
\end{align*}
As a result, we have 
\[
\nu_{\text{up}}=\max_{j}\left(x_{j}^{t}\alpha_{j}-\frac{\beta_{j}^{t}}{e_{j}}\right)
\]
To improve numerical stability, one could use $\nu_{\text{low}}^{\prime}=2n\max_{j}x_{j}^{t}\alpha_{j}-\min_{j}\frac{\beta_{j}^{t}}{e_{j}}$
and $\nu_{\text{up}}^{\prime}<\max_{j}\frac{x_{j}^{t}\alpha_{j}}{2}-\frac{\beta_{j}^{t}}{e_{j}}$. 

\subsection{Case 2: QU}

For QU, we aim to solve equation \eqref{eq:dichotomy-algo-QU} for $\nu$:
\[
h_{2}\left(\nu\right)=\sum_{j}e_{j}\max\left(\frac{-\beta_{j}^{t}-\nu e_{j}+\sqrt{\left(\beta_{j}^{t}+\nu e_{j}\right)^{2}+4\alpha\zeta_{j}^{t}}}{2\alpha},\epsilon\right)-1=0.
\]
Let's analyze the function $h_{2}$. We observe that the function $-x + \sqrt{x^2 + \delta}$ is strictly decreasing over $\mathbb{R}$ for $\delta > 0$, since its derivative $-1 + \frac{x}{\sqrt{x^2 + \delta}}$ is strictly negative for $\delta > 0$. 
Therefore, each term of the sum is decreasing, and $h_{2}$ is a decreasing function.
Note that once at least for one $j$, we have $$-\beta_{j}^{t} - \nu e_{j} + \sqrt{(\beta_{j}^{t} + \nu e_{j})^2 + 4\alpha\zeta_{j}^{t}} \geq 2\alpha\epsilon,$$ and therefore the function $h_{2}$ becomes strictly decreasing. 
In the limit, we have $\lim_{\nu \to -\infty} h_{2}(\nu) = \infty$ and $\lim_{\nu \to \infty} h_{2}(\nu) = \epsilon\|\bm{e}\|_{1} - 1$.
Assuming $\epsilon\|\bm{e}\|_{1}>1$ (which ensures the feasibility of the constraints $\bm{x} \geq \epsilon$ and $\bm{e}^\top \bm{x} = 1$), we know that the function $h_{2}$ has exactly one root.

\paragraph{Negative bound}
Let's find $\nu_{\text{low}}$ such that $h_{2}(\nu) \leq 0$ for $\nu \geq \nu_{\text{low}}$. We start by bounding the term
\[
-\beta_{j}^{t}-\nu e_{j}+\sqrt{\left(\beta_{j}^{t}+\nu e_{j}\right)^{2}+4\alpha\zeta_{j}^{t}}\leq\delta_{j},
\]
where $\delta_{j} > \epsilon > 0$. We move $-\beta_{j}^{t} - \nu e_{j}$ to the left and square the inequality to remove the square root:
\[
\left(\beta_{j}^{t}+\nu e_{j}\right)^{2}+4\alpha\zeta_{j}^{t}\leq\left(\beta_{j}^{t}+\nu e_{j}+\epsilon_{j}\right)^{2}
\]
Eventually, we can extract a bound for $\nu$ to ensure that the inequality is satisfied for a chosen $\delta_{j}$:
\[
\nu\geq\frac{4\alpha\zeta_{j}^{t}-\epsilon_{j}^{2}}{2\delta_{j}e_{j}}-\frac{\beta_{j}^{t}}{e_{j}}.
\]
Let's set $\delta_{j} = \frac{2\alpha}{me_{j}}$, where $m$ is the number of elements in the sum, and take the maximum over $j$ to obtain the bound:
\[
\nu_{\text{low}}\geq\max_{j}\left(m\zeta_{j}^{t}-\frac{\alpha}{me_{j}^{2}}-\frac{\beta_{j}^{t}}{e_{j}}\right)
\]
We observe the validity of this bound provided that $\delta_{j} = \frac{2\alpha}{me_{j}} \geq \epsilon$ for all $j$. Then, it can be verified that
\begin{align*}
h_{2}\left(\nu_{\text{low}}\right) & =\sum_{j}e_{j}\max\left(\frac{-\beta_{j}^{t}-\nu_{\text{low}}e_{j}+\sqrt{\left(\beta_{j}^{t}+\nu_{\text{low}}e_{j}\right)^{2}+4\alpha\zeta_{j}^{t}}}{2\alpha},\epsilon\right)-1\\
 & =\sum_{j}e_{j}\frac{-\beta_{j}^{t}-\nu_{\text{low}}e_{j}+\sqrt{\left(\beta_{j}^{t}+\nu_{\text{low}}e_{j}\right)^{2}+4\alpha\zeta_{j}^{t}}}{2\alpha}-1\\
 & \leq\sum_{j}e_{j}\frac{\delta_{j}}{2\alpha}-1\\
 & =\sum_{j}^{m}e_{j}\frac{\frac{2\alpha}{me_{j}}}{2\alpha}-1=0,
\end{align*}
which proves that $\nu_{\text{low}}$ is a valid negative bound.

\paragraph{Positive bound}
Similarly, let's find $\nu_{\text{up}}$ such that $h_{2}(\nu) \geq 0$ for $\nu \leq \nu_{\text{up}}$. 
This time, we will bound $h_{2}$ from below and obtain
\begin{align*}
h_{2}\left(\nu\right) & =\sum_{j}e_{j}\max\left(\frac{-\beta_{j}^{t}-\nu e_{j}+\sqrt{\left(\beta_{j}^{t}+\nu e_{j}\right)^{2}+4\alpha\zeta_{j}^{t}}}{2\alpha},\epsilon\right)-1\\
 & \geq\sum_{j}e_{j}\frac{-\beta_{j}^{t}-\nu e_{j}+\sqrt{\left(\beta_{j}^{t}+\nu e_{j}\right)^{2}+4\alpha\zeta_{j}^{t}}}{2\alpha}-1\\
 & \geq\frac{1}{2\alpha}\sum_{j}e_{j}\left(-\beta_{j}^{t}-\nu e_{j}\right)-1\\
 & =-\frac{1}{2\alpha}\left(\sum_{j}e_{j}\beta_{j}^{t}+\nu\sum_{j}e_{j}^{2}\right)-1>0.
\end{align*}
We can, therefore, define
\[
\nu_{\text{up}}=-\frac{2\alpha+\sum_{j}e_{j}\beta_{j}^{t}}{\sum_{j}e_{j}^{2}}.
\]

\bibliographystyle{plain}
\bibliography{biblio}

\end{document}